\documentclass[letterpaper, 10 pt, conference]{ieeeconf}  % Comment this line out if you need a4paper

\IEEEoverridecommandlockouts                              % This command is only needed if
\usepackage{graphicx} % for pdf, bitmapped graphics files
\usepackage{graphics}
\usepackage{epsfig} % for postscript graphics files
\usepackage{xcolor}
\usepackage{hyperref}

\usepackage{caption}
\usepackage{subcaption}
\usepackage[utf8]{inputenc} % allow utf-8 input
\usepackage[T1]{fontenc}    % use 8-bit T1 fonts
\usepackage{hyperref}       % hyperlinks
\usepackage{url}            % simple URL typesetting
\usepackage{booktabs}       % professional-quality tables

\usepackage{times} % assumes new font selection scheme installed
\usepackage{amsmath} % assumes amsmath package installed
\usepackage{amssymb}  % assumes amsmath package installed
\usepackage{algorithm}
\usepackage{algpseudocode}
\usepackage{tabularx}
\usepackage{tabulary}
\usepackage{enumerate}
\usepackage{mathrsfs}
\usepackage{bbm}
\usepackage{hyphenat}
\usepackage{ulem}
\usepackage{multirow,hhline,color}
\usepackage{dsfont}

\usepackage[sort]{cite}
\hypersetup{breaklinks=true,colorlinks=true,linkcolor=blue,citecolor=blue}
\usepackage{tcolorbox}
\usepackage[noabbrev,capitalise,nameinlink]{cleveref}

\DeclareMathOperator{\argmax}{arg\,max}

\newtheorem{theorem}{\textbf{Theorem}}

\newtheorem{assumption}{\textbf{Assumption}}
\newtheorem{hmm_assumption}{\textbf{C}}
\newtheorem{asynch_assumption}{A}

\newtheorem{remark}{\textbf{Remark}}

\newtheorem{lemma}{\textbf{Lemma}}
\newtheorem{definition}{\textbf{Definition}}
\newtheorem{proposition}{\textbf{Proposition}}

\newcommand*{\QEDA}{\hfill\ensuremath{\blacksquare}}%

\title{Hidden Markov Model Estimation-Based Q-learning for Partially Observable Markov Decision Process\thanks{Research supported by NSF NRI initiative \#1528036.}}

\author{Hyung-Jin Yoon, Donghwan Lee, and Naira Hovakimyan% <-this % stops a space
\thanks{Hyung-Jin Yoon and Naira Hovakimyan are with the Department of Mechanical Science and Engineering, University of Illinois at Urbana-Champaign (UIUC), Urbana, IL 61801, USA. Donghwan Lee is with the Department of Industrial and Enterprise Systems Engineering in UIUC.
        {\tt\small \{hyoon33, nhovakim, donghwan\}@illinois.edu}}
}

\begin{document}

\maketitle
\thispagestyle{empty}
\pagestyle{empty}

\begin{abstract}
The objective is to study an on-line Hidden Markov model (HMM) estimation-based Q-learning algorithm for partially observable Markov decision process (POMDP) on finite state and action sets. When the full state observation is available, Q-learning finds the optimal action-value function given the current action (Q-function). However, Q-learning can perform poorly when the full state observation is not available. In this paper, we formulate the POMDP estimation into a HMM estimation problem and propose a recursive algorithm to estimate both the POMDP parameter and Q-function concurrently. Also, we show that the POMDP estimation converges to a set of stationary points for the maximum likelihood estimate, and the Q-function estimation converges to a fixed point that satisfies the Bellman optimality equation weighted on the invariant distribution of the state belief determined by the HMM estimation process.
\end{abstract}

%%%%%%%%%%%%%%%%%%%%%%%%%%%%%%%%%%%%%%%%%%%%%%%%%%%%%%%%%%%%%%%%%%%%%%%%%%%%%%%%
\section{Introduction}
Reinforcement learning (RL) is getting significant attention due to the recent successful demonstration of the `Go game', where the RL agents outperform humans in certain tasks (video game~\cite{mnih2015human}, playing Go~\cite{silver2017mastering}). Although the demonstration shows the great potential of the RL, those game environments are confined and restrictive compared to what ordinary humans go through in their everyday life. One of the major differences between the game environment and the real-life is the presence of unknown factors, i.e. the observation of the state of the environment is incomplete. Most RL algorithms are based on the assumption that complete state observation is available, and the state transition depends on the current state and the action (Markovian assumption). Markov decision process (MDP) is a modeling framework with the Markovian assumption. Development and analysis of the standard RL algorithm are based on MDP. Applying those RL algorithms with incomplete observation may lead to poor performance. In~\cite{singh1994learning}, the authors showed that a standard policy evaluation algorithm can result in an arbitrary error due to the incomplete state observation. In fact, the RL agent in~\cite{mnih2015human} shows poor performance for the games, where inferring the hidden context is the key for winning.

Partially observable Markov decision process (POMDP) is a generalization of MDP that incorporates the incomplete state observation model.  When the model parameter of a POMDP is given, the optimal policy is determined by using dynamic programming on the belief state of MDP, which is transformed from the POMDP~\cite{lovejoy1991survey}. The belief state of MDP has continuous state space, even though the corresponding POMDP has finite state space. Hence, solving a dynamic programming problem on the belief state of MDP is computationally challenging. There exist a number of results to obtain approximate solutions to the optimal policy, when the model is given,~\cite{littman1995learning, yu2004discretized}. When the model of POMDP is not given (model-free), a choice is in the policy gradient approach without relying on Bellman's optimality. For example, Monte-Carlo policy gradient approaches~\cite{williams1992simple, bartlett2002estimation} are known to be less vulnerable to the incomplete observation, since they do not require to learn the optimal action-value function, which is defined using the state of the environment. However, the Monte-Carlo policy gradient estimate has high variance so that convergence to the optimal policy typically takes longer as compared to other RL algorithms, which utilize  Bellman's optimality principle when the full state observation is available.

A natural idea is to use a dynamic estimator of the hidden state and apply the optimality principle to the estimated state. Due to its universal approximation property, the recurrent neural networks (RNN) are used to incorporate the estimation of the hidden state in reinforcement learning. In~\cite{hausknecht2015deep}, the authors use an RNN to approximate the optimal value state function using the memory effect of the RNN. In~\cite{heess2015memory}, the authors propose an actor-critic algorithm, where RNN is used for the critic that takes the sequential data. However, the RNNs in~\cite{hausknecht2015deep, heess2015memory} are trained only based on the Bellman optimality principle, but do not consider how accurately the RNNs can estimate the state which is essential for applying  Bellman optimality principle. Without reasonable state estimation, taking an optimal decision even with given correct optimal action-value function is not possible. To the best of the authors' knowledge, most RNNs used in reinforcement learning do not consider how the RNN accurately infers the hidden state.

In this paper, we aim to develop a recursive estimation algorithm for a POMDP to estimate the parameters of the model, predict the hidden state, and also determine the optimal value state function concurrently. The idea of using a recursive state predictor (Bayesian state belief filter) in RL was investigated in~\cite{chrisman1992reinforcement, ross2008bayes, karkus2017qmdp, guo2016pac}. In~\cite{chrisman1992reinforcement}, the author proposed to use the Bayesian state belief filter for the estimation of the Q-function.  In~\cite{ross2008bayes}, the authors implemented the Bayesian state belief update with an approximation technique for the ease of computation and analyzed its convergence. More recently, the authors in~\cite{karkus2017qmdp} combine the Bayesian state belief filter and QMDP~\cite{littman1995learning}. However, the algorithms in~\cite{chrisman1992reinforcement, ross2008bayes, karkus2017qmdp} require the POMDP model parameter readily available\footnote{In \cite{ross2008bayes}, the algorithm needs full state observation for the system identification of POMDP.}.  A \emph{model-free} reinforcement learning that uses HMM formulation is presented in~\cite{guo2016pac}. The result in~\cite{guo2016pac} shares the same idea as ours, where we use HMM estimator with a fixed behavior policy, in order to disambiguate the hidden state, learn the POMDP parameters, and find optimal policy. However, the algorithm in~\cite{guo2016pac} involves multiple phases, including identification and design, which are hard to apply  online to real-time learning tasks, whereas recursive estimation is more suitable (e.g., DQN, DDPG, or Q-learning are online algorithms). The main contribution of this paper is to present and analyze a new \emph{on-line} estimation algorithm to simultaneously estimate the POMDP model parameters and corresponding optimal action-value function (Q-function), where we employ online HMM estimation techniques~\cite{krishnamurthy2002recursive, legland1997recursive}.

The remainder of the paper is organized as follows. In Section II, HMM interpretation of the POMDP with a behavior policy presented. In Section III, the proposed recursive estimation of the HMM, POMDP, and Q-function is presented and the convergence of the estimator is analyzed. In Section IV, a numerical example is presented. Section V summarizes.

%%%%%%%%%%%%%%%%%%%%%%%%%%%%%%%%%%%%%%%%%%%%%%%%%%%%%%%%%%%%%%%%%%%%%%%%%%%%%%%%
\section{A HMM: POMDP excited by Behavior policy}\label{sec: A HMM}
We consider a partially observable Markov decision process (POMDP) on finite state and action sets. A fixed behavior policy\footnote{Behavior policy is the terminology used in the reinforcement learning, and it is analogous to excitation of a plant for system identification.} excites the POMDP so that all pairs of state-action are realized infinitely often along the infinite time horizon.

\subsection{POMDP on finite state-action sets}
The POMDP $(\mathcal{S}, \mathcal{A}, T_a(s,s'), R(s,a), \mathcal{O}, O(o, s), \gamma)$ comprises: a finite state space $\mathcal{S}:=\{1, \dots, I\}$, a finite action space $\mathcal{A}:=\{1,\dots,K\}$, a state transition probability $T_a(s,s')=P(s_{n+1}=s'|s_n=s, a_n=a)$, for $s,s' \in \mathcal{S}$ and $a\in\mathcal{A}$, a reward model $R\in\mathbb{R}$ such that $R(s,a) = r(s,a) + \delta$, where $\delta$ denotes independent identically distributed (i.i.d.) Gaussian noise $\delta \sim \mathcal{N}(0, \sigma^2)$, a finite observation space $\mathcal{O}:=\{1, \dots, J\}$, an observation probability $O(o, s) = P(o_n = o|s_n=s)$, and the discount factor $\gamma \in [0,1)$. At each time step $n$, the agent first observes $o_n \in \mathcal{O}$ from the environment at the state $s_n \in \mathcal{S}$, does action $a_n\in\mathcal{A}$ on the environment and gets the reward $r_n \in \mathbb{R}$ in accordance to $R(s,a)$.

\subsection{Behavior policy and HMM}
A behavior policy is used to estimate the model parameters. Similarly to other off-policy reinforcement learning (RL) algorithms, i.e. Q-learning~\cite{watkins1992q}, a behavior policy excites the POMDP, and the estimator uses the samples generated from the controlled POMDP. The behavior policy's purpose is system identification (in other words, estimation of the POMDP parameter). We denote the behavior policy by $\mu$, which is a conditional probability, i.e. $\mu(o)=P(a|o)$. Since we choose how to excite the system, the behavior policy can be used in the estimation. The POMDP with $\mu(o)$ becomes a hidden Markov model (HMM), as illustrated in Fig.~\ref{fig:Diagram}.
\begin{figure}[thpb]
\centering
 \includegraphics[width=0.48\textwidth]{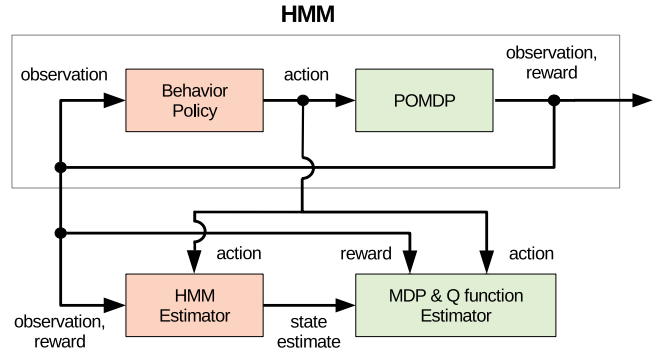}
 \caption{A POMDP Estimation Framework.}
\medskip
\label{fig:Diagram}
\end{figure}

The HMM comprises: state transition probability $P(s_{n+1}=s'|s_n=s) = P(s_{n+1}=s'|s_n=s, a_n=a; \mu, O)$ for all pairs of $(s, s')$ and the extended observation probability, i.e. $P(o, a, r|s)$ which is determined by the POMDP model parameters: $O(o, s)$, $R(s,a)$ and the behavior policy $\mu(o)$.

For the ease of notation, we define the following tensor and matrices: $\mathbf{T}\in\mathbb{R}^{K \times I \times I}$ such that $\mathbf{T}_{ijk}=P(s_{n+1}=k|s_n=j, a_n=i)$, $\mathbf{R}\in\mathbb{R}^{K \times I}$ such that $\mathbf{R}_{ij}=r(s=j, a=i)$, $\mathbf{O}\in\mathbb{R}^{I \times J}$ such that $\mathbf{O}_{ij} = P(o_n = j|s_n=i)$, and $\mathbf{P}\in\mathbb{R}^{I \times I}$ such that $\mathbf{P}_{ij}=P(s_{n+1}=j|s_n=j ; \mu )$.

The HMM estimator in Fig.~\ref{fig:Diagram} learns the model parameters $\mathbf{P}, \mathbf{O}, \mathbf{R}, \sigma$, where $\sigma$ is defined in II-A, and also provides the state estimate (or belief state) to the MDP and Q-function estimator. Given the transition of the state estimates and the action, the MDP estimator learns the transition model parameter $\mathbf{T}$. Also, the optimal action-value function $Q^*(s, a)$ is recursively estimated based on the transition of the state estimates, reward sample and the action taken.

%%%%%%%%%%%%%%%%%%%%%%%%%%%%%%%%%%%%%%%%%%%%%%%%%%%%%%%%%%%%%%%%%%%%%%%%%%%%%%%%
\section{HMM Q-Learning Algorithm For POMDPs}
The objective of this section is to present a new HMM model estimation-based Q-learning algorithm, called HMM Q-learning, for POMDPs, which is the main outcome of this paper. The pseudo code of the recursive algorithm is in Algorithm~\ref{alg: HMM Q-Learning}.
\begin{algorithm}[thpb]
%=====================================================================================
\caption{HMM Q-Learning}
\begin{algorithmic}[1]
%=====================================================================================

\State Set $n=0$.
\State Observe $o_0$ from the environment.
\State Initialize: the parameter $(\theta_0, Q_0, T_0)$, the states $(\mathbf{u}_0, \omega_0)$, $\hat{p}_n^{(\text{prev})} \in \mathcal{P}(\mathcal{S})$ as uniform distribution, randomly choose $a_n^{(\text{prev})} \in \mathcal{A}$, and set $r_n^{(\text{prev})} = 0$.

\Repeat
\State Act $a$ with $\mu(o_n) = P(a|o_n)$, get reward $r$ and the next observation $o'$ from the environment.
\State Use $y_n = (o_n, a, r)$ and $(\theta_{n}, \mathbf{u}_n, \omega_n)$ to update the estimator as follows:
\begin{equation*}
\begin{aligned}
\theta_{n+1} &= \Pi_{H} \left[ \theta_{n} + \epsilon_n \mathbf{S}\left(y_n, \mathbf{u}_n, \mathbf{\omega}_n ;\theta_n\right) \right],\\
\mathbf{u}_{n+1} &= f(y_n, \mathbf{u}_n;\theta_n),\\
\mathbf{\omega}^{(l)}_{n+1} &= \Phi(y_n, \mathbf{u}_{n};\theta_n)\mathbf{\omega}^{(l)}_n+\frac{\partial f(y_n, \mathbf{u}_{n};\theta_n)}{\partial \theta^{(l)}},
\end{aligned}
\end{equation*}
where
\begin{equation*}
\begin{aligned}
&f(y_n, \mathbf{u}_n;\theta_n)\triangleq \frac{\mathbf{P}_{\theta_n}^\top\mathbf{B}(y_n;\theta_n)\mathbf{u}_n}{\mathbf{b}^\top(y_n;\theta_n)\mathbf{u}_n},\\
&\mathbf{S}\left(y_n, \mathbf{u}_n, \mathbf{\omega}_n;\theta_n\right) =\frac{\partial  \log \left(\mathbf{b}^\top(y_n;\theta_n)\mathbf{u}_n\right)}{\partial \theta},
\end{aligned}
\end{equation*}
$\Pi_H$ denotes the projection on the convex constraint set $H\subseteq\Theta$, $\epsilon_n \geq 0$ denotes the step size, $\mathbf{\omega}_n \in \mathbb{R}^{I\times L}$ denotes the Jacobian of the state prediction vector $\mathbf{u}_n$ with respect to the parameter vector $\theta_n$.

\State Calculate $\hat{p}_n := [P(s=i|y_n, \mathbf{u}_n;\theta_n)]_{i\in\mathcal{I}}$ as in~\eqref{eq:state_est_F_n}.

\State Calculate $\hat{p}(s_{n-1},s_{n})$ with $\hat{p}_n^{(\text{prev})}$ and $\hat{p}_n$ as in~\eqref{eq:state_transition_est}.

\State Use $r_n^{\text{prev}}$, $a_n^{\text{prev}}$ and $\hat{p}(s_{n-1},s_{n})$ to update $Q_n$ according to \eqref{eq: Q_estimator_with_HMM}.

\State Use $\hat{p}(s_{n-1},s_{n})$ to update $T_n$ according to \eqref{eq: transition_estimation}.

\State $(\hat{p}_n^{(\text{prev})}, r_n^{\text{prev}}, a_n^{\text{prev}})  \leftarrow (\hat{p}_n$, r, a).

\State $o_n \leftarrow o'$.

\State $n \leftarrow n+1$.

\Until{a certain stopping criterion is satisfied.}

%=====================================================================================
\end{algorithmic}
%=====================================================================================
\label{alg: HMM Q-Learning}
\end{algorithm}
It recursively estimates the maximum likelihood estimate of the POMDP parameter and Q-function using partial observation. The recursive algorithm integrates (a) the HMM estimation, (b) MDP transition model estimation, and (c) the Q-function estimation steps. Through the remaining subsections, we prove the convergence of~Algorithm~\ref{alg: HMM Q-Learning}. To this end, we first make the following assumptions.
\begin{assumption}\label{assumption:nice_behavioral_policy}
The transition probability matrix $\mathbf{P}$ determined by the transition $\mathbf{T}$, the observation $\mathbf{O}$, and the behavior policy $\mu(o)$ are aperiodic and irreducible~\cite{norris1998markov}. Furthermore, we assume that the state-action pair visit probability is strictly positive under the behavior policy.
\end{assumption}
We additionally assume the following. 
\begin{assumption}\label{assumption:positive observation probability}
All elements in the observation probability matrix $\mathbf{O}$ are strictly positive, i.e. $\mathbf{O}_{i,j} > 0$ for all $i\in\mathcal{S}$ and $j\in\mathcal{O}$.
\end{assumption}
Under these assumptions, we will prove the following convergence result.
\begin{proposition}[Main convergence result]\label{proposition: convergence of overall algoritm}
Suppose that Assumption~\ref{assumption:nice_behavioral_policy} and Assumption~\ref{assumption:positive observation probability} hold. Then the following statements are true:

(i) The iterate $\theta_n$ in~
Algorithm~\ref{alg: HMM Q-Learning} converges almost surely to the stationary point $\theta^*$ of the conditional log-likelihood density function based on the sequence of the extended observations $\{y_i = (o_i, r_i, a_i)\}_{i=0}^n$, $l_n(\theta) = \frac{1}{n+1}\log p_n(y_0, y_1, \dots, y_n|s_0, s_1, \dots, s_n;\theta)$,
i.e., the point $\theta$ is satisfying
\begin{equation*}
E\left[\frac{\partial  \log \left(\mathbf{b}^\top(y_n;\theta)\mathbf{u}_n\right)}{\partial \theta}\right] \in N_{H}(\theta),
\end{equation*}
where $N_{H}(\theta)$ is the normal cone~\cite[pp.~343]{bertsekas1999nonlinear} of the convex set $H$ at $\theta \in H$, and the expectation $E$ is taken with respect to the invariant distribution of $y_n$ and $\mathbf{u}_n$.

(ii) Define $\bar{p}(s,s'):= \lim_{n \to \infty} \hat{p}(s_{n-1},s_n)$ in the almost sure convergence sense. Then the iterate $\{Q_n\}$ in Algorithm~\ref{alg: HMM Q-Learning}
converges in distribution to the optimal Q-function $\hat Q^*$, satisfying \begin{equation*}
\begin{aligned}
&\hat Q^*(s,a)= \sum_{s'}\bar{p}(s,s')\left(r(s,a) + \gamma \max_{a'}\hat Q^*(s',a') \right).
\end{aligned}
\end{equation*}
\end{proposition}

\subsection{HMM Estimation} \label{subsection: HMM Estimation}
We employ the recursive estimators of HMM from~\cite{krishnamurthy2002recursive, legland1997recursive} for our estimation problem, where we estimate the true parameter $\theta^*$ with the model parameters $(\mathbf{P}, \mathbf{R}, \mathbf{O}, \sigma)$ being parametrized as continuously differentiable functions of the vector of real numbers $\theta \in \Theta \subset \mathbb{R}^L$, such that $\theta^*\in\Theta$ and $(\mathbf{P}_{\theta^*}, \mathbf{R}_{\theta^*}, \mathbf{O}_{\theta^*}, \sigma_{\theta^*}) = (\mathbf{P}, \mathbf{R}, \mathbf{O}, \sigma)$. We denote the functions of the parameter as  $(\mathbf{P}_\theta, \mathbf{R}_\theta, \mathbf{O}_\theta, \sigma_\theta)$ respectively. In this paper, we  consider the normalized exponential function (or softmax function)\footnote{Let $\{\alpha_{1,1}, \dots, \alpha_{I,I}\} $ denote the parameters for the probability matrix $\mathbf{P}_\theta$. Then the $(i,j)$\textsuperscript{th} element of $\mathbf{P}_\theta$ is $\frac{\exp(\alpha_{i,j})}{\sum_{j'=1}^{I} \exp(\alpha_{i,j'})}$.
} to parametrize the probability matrices $\mathbf{P}_\theta$, $\mathbf{O}_\theta$. The reward matrix $\mathbf{R}_\theta$ is a matrix in $\mathbb{R}^{I \times K}$ and $\sigma_\theta$ is a scalar.

The iterate $\theta_n$ of the recursive estimator converges to the set of the stationary points, where the gradient of the likelihood density function is zero~\cite{krishnamurthy2002recursive, legland1997recursive}. The conditional log-likelihood density function based on the sequence of the extended observations  $\{y_i = (o_i, r_i, a_i)\}_{i=0}^n$ is
\begin{equation}\label{eq: log-likelihood}
l_n(\theta) = \frac{1}{n+1}\log p_n(y_0, y_1, \dots, y_n|s_0, s_1, \dots, s_n;\theta).
\end{equation}
When the state transition and observation model parameters are available, the state estimate
\begin{equation}\label{eq: state-estimate}
\mathbf{u}_n = [u_{n,1}, u_{n,2}, \dots, u_{n,I}]^\top,
\end{equation}
where
$
u_{n,i} = P(s_n=i|y_0, y_1, \dots, y_n; \theta)
$
is calculated  from the recursive state predictor (Bayesian state belief filter)~\cite{baum1970maximization}. The state predictor is given as follows:
\begin{equation}\label{eq: Bayesian-filter}
\mathbf{u}_{n+1} = \frac{\mathbf{P}_{\theta}^\top\mathbf{B}(y_n;\theta)\mathbf{u}_n}{\mathbf{b}^\top(y_n;\theta)\mathbf{u}_n},
\end{equation}
where
\begin{equation}\label{eq: Output-likelihood}
\mathbf{b}(y_n;\theta) = [b_1(y_n;\theta), b_2(y_n;\theta), \dots, b_I(y_n;\theta)]^\top,
\end{equation}
\begin{align*}
b_i(y_n;\theta) &= p(y_n|s_n = i ;\theta) \\
                &= P(o_n|s_n=i; \theta)P(a_n|o_n)p(r_n|s_n=i, a_n; \theta),
\end{align*}
and $\mathbf{B}(y_n;\theta)$ is the diagonal matrix with $\mathbf{b}(y_n;\theta)$.
Using Markov property of the state transitions and the conditional independence of the observations given the states, it is easy to show that the conditional likelihood density~\eqref{eq: log-likelihood} can be expressed with the state prediction $\mathbf{u}_n(\theta)$ and the observation likelihood $\mathbf{b}(y_n;\theta)$ as follows~\cite{krishnamurthy2002recursive,legland1997recursive}:
\begin{equation}\label{eq: log-likelihood-state-estimate}
l_n(\theta) = \frac{1}{n+1}\sum_{k=0}^n \log \left(\mathbf{b}^\top(y_n;\theta)\mathbf{u}_n\right).
\end{equation}
\begin{remark}
Since the functional parameterization of $(\mathbf{P}_\theta, \mathbf{R}_\theta, \mathbf{O}_\theta, \sigma_\theta)$ uses the non-convex soft-max functions, $l(\theta)$ is non-convex in general.
\end{remark}

Roughly speaking, the recursive HMM estimation~\cite{krishnamurthy2002recursive,legland1997recursive} calculates the online estimate of the gradient of $l_n(\theta_n)$ based on the current output $y_n$, the state prediction $\mathbf{u}_n(\theta_n)$, and the current parameter estimate $\theta_n$ and adds the stochastic gradient to the current parameter estimate $\theta_n$, i.e. it is a stochastic gradient \emph{ascent} algorithm to maximize the conditional likelihood.

We first introduce the HMM estimator~\cite{krishnamurthy2002recursive,legland1997recursive} and then apply the convergence result~\cite{krishnamurthy2002recursive} to our estimation task. The recursive HMM estimation in Algorithm~\ref{alg: HMM Q-Learning} is given by:
\begin{equation}\label{eq: HMM_estimator}
    \theta_{n+1} = \Pi_{H} \left[ \theta_{n} + \epsilon_n \mathbf{S}\left(y_n, \mathbf{u}_n, \mathbf{\omega}_n ;\theta_n\right) \right],
\end{equation}
\begin{equation}\label{eq: Score_estimate}
        \mathbf{S}\left(y_n, \mathbf{u}_n, \mathbf{\omega}_n;\theta_n\right) =\frac{\partial  \log \left(\mathbf{b}^\top(y_n;\theta_n)\mathbf{u}_n\right)}{\partial \theta},
\end{equation}
where $\Pi_H$ denotes the projection onto the convex constraint set $H\subseteq\Theta$, $\epsilon_n \geq 0$ denotes the diminishing step-size such that $\epsilon_n \rightarrow 0, \; \sum_n \epsilon_n = \infty$, $\mathbf{\omega}_n \in \mathbb{R}^{I\times L}$ denotes the Jacobian  of  the state prediction vector $\mathbf{u}_n$ with respect to the parameter vector $\theta_n$.
\begin{remark}
(i) The diminishing step-size used above is standard in the stochastic approximation algorithms (see Chapter 5.1 in~\cite{kushner2003stochastic}). (ii) The algorithm with a projection on to the constraint convex set $H$ has advantages such as guaranteed stability and convergence of the algorithm, preventing numerical instability (e.g. floating point underflow) and avoiding exploration in the parameter space far away from the true one. The useful parameter values in a properly parametrized practical problem are usually confined by constraints of physics or economics to some compact set~\cite{kushner2003stochastic}. $H$ can be usually determined based on the solution analysis depending on the problem structure.
\end{remark}

Using Calculus, the equation~\eqref{eq: Score_estimate} is written in terms of $\mathbf{u}_n$, $\mathbf{\omega}_n$, $\mathbf{b}(y_n;\theta_n)$, and its partial derivatives as follows:
\begin{equation*}
\mathbf{S}\left(y_n, \mathbf{u}_n, \mathbf{\omega}_n; \theta_n\right) =
\begin{bmatrix}
   S^{(1)}\left(y_n, \mathbf{u}_n, \mathbf{\omega}_n;\theta_n\right) \\
   S^{(2)}\left(y_n, \mathbf{u}_n, \mathbf{\omega}_n;\theta_n\right) \\
                                                                                 \vdots                                       \\
   S^{(L)}\left(y_n, \mathbf{u}_n, \mathbf{\omega}_n;\theta_n\right) \\
\end{bmatrix},
\end{equation*}
\begin{equation}\label{eq: S_n}
\begin{aligned}
   &S^{(l)}\left(y_n, \mathbf{u}_n, \mathbf{\omega}_n;\theta_n\right)  \\
   %================================================================
   &= \frac{\mathbf{b}^\top(y_n;\theta_n) \mathbf{\omega}^{(l)}_n}{\mathbf{b}^\top(y_n;\theta_n)\mathbf{u}_n}
        +
        \frac{\left((\partial/\partial \theta^{(l)}) \mathbf{b}^\top(y_n;\theta_n)\right) \mathbf{u}_n}{\mathbf{b}^\top(y_n;\theta_n)\mathbf{u}_n},
\end{aligned}
\end{equation}
where $\mathbf{\omega}^{(l)}_n$ is the $l$\textsuperscript{th} column of the $\mathbf{\omega}_n \in \mathbb{R}^{I \times L}$,
$\mathbf{u}_n(\theta_n)$ is recursively updated using the state predictor in~\eqref{eq: Bayesian-filter} as
\begin{equation}\label{eq: update-u}
\mathbf{u}_{n+1} = \frac{\mathbf{P}_{\theta_n}^\top\mathbf{B}(y_n;\theta_n)\mathbf{u}_n}{\mathbf{b}^\top(y_n;\theta_n)\mathbf{u}_n} \triangleq f(y_n, \mathbf{u}_n;\theta_n),
\end{equation}
with $\mathbf{u}_0$ being initialized as an arbitrary distribution on the finite state set, $\mathbf{P}_{\theta_n}$ being the state transition probability matrix for the current iterate $\theta_n$. The state predictor~\eqref{eq: update-u} calculates the state estimate (or Bayesian belief) on the $s_{n+1}$ by normalizing the conditional likelihood $p(y_n|s_n = i;\theta_n)P(s_n=i|y_0,\dots, y_n)$ and then multiplying it with the state transition probability $P(s_{n+1}=j|s_n=i;\theta_n)$. The predicted state estimate is used recursively to calculate the state prediction in the next step. Taking derivative on the update law~\eqref{eq: update-u}, the update law for $\mathbf{\omega}^{(l)}_n$ is
\begin{equation}\label{eq: update-omega}
    \mathbf{\omega}^{(l)}_{n+1}=\Phi(y_n, \mathbf{u}_{n};\theta_n)\mathbf{\omega}^{(l)}_n+\frac{\partial f(y_n, \mathbf{u}_{n};\theta_n)}{\partial \theta^{(l)}},
\end{equation}
where
{\small
\begin{equation*}
\Phi(y_n, \mathbf{u}_{n};\theta_n) = \frac{\mathbf{P}_{\theta_n}^\top\mathbf{B}(y_n;\theta_n)}{\mathbf{b}^\top(y_n;\theta_n)\mathbf{u}_n}
\left(\mathbf{I} - \frac{\mathbf{u}_n\mathbf{b}^\top(y_n;\theta_n)}{\mathbf{b}^\top(y_n;\theta_n)\mathbf{u}_n} \right),
\end{equation*}
\begin{equation*}
\begin{aligned}
%================================================================
&\frac{\partial f(y_n, \mathbf{u}_{n};\theta_n)}{\partial \theta^{(l)}}\\
& = \mathbf{P}_{\theta_n}^\top \left(\mathbf{I} - \frac{\mathbf{B}(y_n;\theta_n) \mathbf{u}_n \mathbf{e}^\top }{\mathbf{b}^\top(y_n;\theta_n)\mathbf{u}_n} \right)\frac{\left(\partial \mathbf{B}(y_n;\theta_n)   /  \partial \theta^{(l)} \right) \mathbf{u}_n}{\mathbf{b}^\top(y_n;\theta_n)\mathbf{u}_n } \\
%================================================================
& + \frac{\left( \partial \mathbf{P}_{\theta_n}^\top   /  \partial \theta^{(l)} \right)\mathbf{B}(y_n;\theta_n) \mathbf{u}_n}{\mathbf{b}^\top(y_n;\theta_n)\mathbf{u}_n},
\end{aligned}
\end{equation*}
}
$\theta^{(l)}$ denotes the $l$\textsuperscript{th} element of the  parameter $\theta_n$, $\mathbf{I}$ denotes the $I \times I$ identity matrix, $\mathbf{e}=[1,\dots, 1]^\top$, the initial $\omega^{(l)}_0$ is arbitrarily  chosen from $\Sigma = \{\omega^{(l)} \in \mathbb{R}^I :e^\top \omega^{(l)}  = 0\}$.

At each time step $n$, the HMM estimator defined by~\eqref{eq: HMM_estimator},~\eqref{eq: S_n},~\eqref{eq: update-u}, and~\eqref{eq: update-omega} updates $\theta_n$ based on the current sample $y_n=(o_n, r_n, a_n)$, while keeping track of the state estimate $\mathbf{u}_n$, and its partial derivative $\omega_n$.

Now we state the convergence of the estimator.
\begin{proposition}\label{proposition: convergence of HMM estimation}
Suppose that Assumption~\ref{assumption:nice_behavioral_policy} and Assumption~\ref{assumption:positive observation probability} hold. Then, the following statements hold:

(i) The extended Markov chain $\{ s_n, y_n, \mathbf{u}_n, \omega_n \}$ is geometrically ergodic\footnote{
A Markov chain with transition probability matrix $\mathbf{P}$ is \emph{geometrically ergodic}, if for finite constants $c_{ij}$ and a $\beta < 1$
\begin{equation*}
    |(\mathbf{P}^n)_{i,j} - \pi_j| \leq c_{ij} \beta^n,
\end{equation*}
where $\pi$ denotes the stationary distribution.
}.

(ii) For $\theta \in \Theta$, the log-likelihood $l_n(\theta)$ in \eqref{eq: log-likelihood} almost surely converges to $l(\theta)$,
\begin{equation}\label{eq: stationary-log-likelihood}
    l(\theta) = \int_{\mathcal{Y}\times\mathcal{P}(\mathcal{S})} \log[ \mathbf{b}^\top(y;\theta)\mathbf{u}] \; \nu(dy, d\mathbf{u}),
\end{equation}
where $\mathcal{Y}:=\mathcal{O} \times \mathbb{R} \times \mathcal{A}$, $\mathcal{P}(\mathcal{S})$ is the set of probability distribution on $\mathcal{S}$, and $\nu(dy, d\mathbf{u})$ is the marginal distribution of $\nu$, which is the invariant distribution of the extended Markov chain.

(iii) The iterate $\{\theta_n\}$ converges almost surely to the invariant set (set of equilibrium points) of the ODE
\begin{equation}\label{eq: ODE_HMM}
    \dot{\theta} = \mathbf{H}(\theta) + \tilde{m}= \Pi_{T_H(\theta)}[\mathbf{H}(\theta)], \quad \theta(0) = \theta_0,
\end{equation}
where $\mathbf{H}(\theta) = E[\mathbf{S}(y_n, \mathbf{u}_n, \mathbf{\omega}_n;\theta)]$, the expectation $E[\cdot]$ is taken with respect to $\nu$, and $\tilde{m}(\cdot)$ is the projection term to keep in $H$, $T_H(\theta)$ is the tangent cone of $H$ at $\theta$~\cite[pp.~343]{bertsekas1999nonlinear}.
\end{proposition}
{
\begin{remark}
The second equation in~\eqref{eq: ODE_HMM} is due to~\cite[Appendix~E]{bhatnagar2012stochastic}. Using the definitions of tangent
and normal cones~\cite[pp.~343]{bertsekas1999nonlinear}, we can readily prove that the set of stationary points of~\eqref{eq: ODE_HMM} is $\{
\theta\in H:\Pi_{T_H(\theta)}(\mathbf{H}(\theta))=0\}=\{ \theta\in H:\mathbf{H}(\theta)\in N_H(\theta)\}$, where $N_H(\theta)$ is the normal cone of $H$ at $\theta \in H$. Note that the set of stationary points is identical to the set of KKT points of the constrained nonlinear programming $\min_{\theta \in H} l(\theta)$.
\end{remark}}
\begin{remark}
Like other maximum likelihood estimation algorithms, further assuming that $l(\theta)$ is concave, it is possible to show the ${\theta_n}$ converges to the unique maximum likelihood estimate. However, the convexity of $l(\theta)$ is not known in prior. Similarly,  asymptotic stability of the ODE~\eqref{eq: ODE_HMM} is assumed to show the desired convergence in~\cite{krishnamurthy2002recursive}. We refer to~\cite{krishnamurthy2002recursive} for the technical details regarding the convergence set.
\end{remark}

\begin{proof}
We employed the convergence result in~\cite{krishnamurthy2002recursive}. We prove that the HMM estimation converges to the invariant set of ODE~\eqref{eq: ODE_HMM} by verifying the assumptions in~\cite{krishnamurthy2002recursive} for the POMDP with the behavior policy described in Section~\ref{sec: A HMM}. See Appendix~\ref{Appendix A.} for the details.
\end{proof}

\subsection{Estimating Q-function with the HMM State Predictor}\label{subsection: Q Estimation}
In addition to estimation of the HMM parameters $(\mathbf{P}, \mathbf{R}, \mathbf{O}, \sigma)$, we aim to \emph{recursively} estimate the optimal action-value function $Q^*(s,a): \mathcal{S}\times\mathcal{A} \rightarrow \mathbb{R}$ using \emph{partial} state observation.

From  Bellman's optimality principle, $Q^*(s,a)$ function is defined as
{\small
\begin{equation}\label{eq:Bellam_Optimality}
    Q^*(s, a)=\sum_{s'}P(s'|s,a)\left(r(s,a) + \gamma \max_{a'}Q^*(s',a') \right),
\end{equation}
}
where $P(s'|s,a)$ is the state transition probability, which corresponds to $T_a(s,s')$ in the POMDP model.
The standard Q-learning from~\cite{watkins1992q} estimates $Q^*(s,a)$ function using the recursive form:
{\small
\begin{equation*}\label{eq:Q_Learning}
\begin{aligned}
    &Q_{n+1}(s_n, a_n)\\
                                  &=Q_{n}(s_n, a_n) + \epsilon_n \left(r_n + \gamma \max_{a'}Q_n(s_{n+1},a') - Q_{n}(s_n, a_n) \right).
\end{aligned}
\end{equation*}
}
Since the state $s_n$ is not directly observed in POMDP, the state estimate $\mathbf{u}_n$ in~\eqref{eq: update-u} from the HMM estimator is used instead of $s_n$. Define the estimated state transition $\hat{p}(s_{n-1},s_{n})$ as
\begin{equation}\label{eq:state_transition_est}
\begin{aligned}
    &\hat{p}(s_{n-1},s_{n})\\
    &= P(s_{n-1}, s_{n}|y_n, y_{n-1}, \mathbf{u}_n, \mathbf{u}_{n-1} ; \theta_n, \theta_{n-1}) \\
    &= P(s_{n-1}|y_{n-1}, \mathbf{u}_{n-1} ; \theta_{n-1})P(s_n|y_n, \mathbf{u}_{n};\theta_n),
\end{aligned}
\end{equation}
where $P(s_n|y_n, \mathbf{u}_n;\theta_n)$ is calculated using Bayes rule:
\begin{equation}\label{eq:state_est_F_n}
    P(s_n=i|y_n, \mathbf{u}_n; \theta_n) = \frac{b_i(y_n) u_{n,i}}{\sum_{j} b_j(y_n) u_{n,j}}.
\end{equation}
Using $\hat{p}(i, j)$ as a surrogate for $P(s'|s,a)$ in \eqref{eq:Bellam_Optimality}, a recursive estimator for $Q^*(s,a)$ is proposed as follows:
{\small
\begin{equation}\label{eq: Q_estimator_with_HMM}
    \begin{aligned}
        &\begin{bmatrix}
            q_{n+1}(1, a_n)\\
            q_{n+1}(2, a_n)\\
            \vdots     \\
            q_{n+1}(I, a_n)
        \end{bmatrix}
        =
        \begin{bmatrix}
            q_{n}(1, a_n)\\
            q_{n}(2, a_n)\\
            \vdots     \\
            q_{n}(I, a_n)
        \end{bmatrix}
        +\\
        &\epsilon_n
        \begin{bmatrix}
            \sum_{j}^I\hat{p}_n(1,j)\left(r_n + \gamma \max_{a'}q_{n}(j,a') - q_{n}(1, a_n) \right) \\
            \sum_{j}^I\hat{p}_n(2,j)\left(r_n + \gamma \max_{a'}q_{n}(j,a') - q_{n}(2, a_n) \right) \\
            \vdots \\
            \sum_{j}^I\hat{p}_n(I,j)\left(r_n + \gamma \max_{a'}q_{n}(j,a') - q_{n}(I, a_n) \right)
        \end{bmatrix},
    \end{aligned}
\end{equation}
}
where $q_n(i,a_n) = Q_n(s=i,a=a_n)$. In the following proposition we establish the convergence of~\eqref{eq: Q_estimator_with_HMM}.

\begin{proposition}\label{proposition: convergence of Q estimation}
Suppose that~Assumption~\ref{assumption:nice_behavioral_policy} and~Assumption~\ref{assumption:positive observation probability} hold. Then the following ODE has a unique globally asymptotically stable equilibrium point:
{\small
\begin{equation*}
        \begin{bmatrix}
            \dot{q}_{1,a}\\
            \dot{q}_{2,a}\\
            \vdots     \\
            \dot{q}_{I,a}
        \end{bmatrix}
        =
        \frac{1}{\bar{u}_a}
        \begin{bmatrix}
        \sum_{j}^I\bar{p}(1,j) (\bar{r} + \gamma \max_{a'}q_{j,a'} - q_{1, a}) \\
        \sum_{j}^I\bar{p}(2,j) (\bar{r} + \gamma \max_{a'}q_{j,a'} - q_{2, a}) \\
        \vdots \\
        \sum_{j}^I\bar{p}(I,j)  (\bar{r} + \gamma \max_{a'}q_{j,a'} - q_{I, a})
        \end{bmatrix},
        \quad a\in {\cal A},
\end{equation*}
}
where $\bar{u}_a$ is determined by the expected frequency of the recurrence to the action $a$ (for the detail, see Appendix~\ref{Appendix B.}), $\bar{p}(i,j)$ denotes the expectation of $\hat{p}(i,j)$, $\bar{r}$ denotes the expectation of $R(s,a)$ and the expectations are taken with the invariant distribution $\nu$. As a result, the iterate $\{Q_n\}$ of the recursive estimation law in~\eqref{eq: Q_estimator_with_HMM} converges in distribution to the unique equilibrium point $\hat Q^*$ of the ODE, i.e., the unique solution of the Bellman equation
{\small
\begin{equation*}
\begin{aligned}
&\hat Q(s, a) = \sum_{s'}\bar{p}(s,s')\left(\bar{r}(s,a) + \gamma \max_{a'} \hat Q(s',a') \right).
\end{aligned}
\end{equation*}
}
\begin{remark}
Note that $\hat p_n$ is the continuous function of the random variables $(y_n, \mathbf{u}_n, \theta_n)$, which almost surely converges due to the ergodicity of the Markov chain $(y_n, \mathbf{u}_n, \omega_n)$ and the  convergence of $\theta_n$ (proven above). By continuous mapping theorem from~\cite{durrett2010probability}, $\hat p_n$ as a continuous function of the converging random variables converges in the same sense.
\end{remark}
\end{proposition}
\begin{proof}
The update of $Q^\epsilon_n$ is asynchronous, as we update the part of $Q_n(s,a)$ for the current action taken. Result on  stochastic approximation from~\cite{kushner2003stochastic} is invoked to prove the convergence. The proof follows from the ergodicity of the underlying Markov chain and the contraction of the operator $HQ =\sum_{s'}\hat{p}(s,s';\theta_L)\left(r(s,a) + \gamma \max_{a'}Q(s',a') \right)$. See Appendix~\ref{Appendix B.} for the details.
\end{proof}

\subsection{Learning State Transition given Action with the HMM State Predictor}\label{subsection: T Estimation}
When the full state observation is available, the transition model   $T_a(s,s')=P(s_{n+1}=s'|s_n=s, a_n=a)$ can be estimated simply counting all the incidents of each transition $(s, a, s')$, and the transition model estimation corresponds to the maximum likelihood estimate. Since the state is partially observed, we use the state estimate instead of  counting transitions.

We aim to estimate the expectation of the following indicator function
\begin{equation}
    T_{s,a,s'} = E[\mathds{1}_{\{s_{n}=s, a_n=j, s_{n+1}=s'\}}],
\end{equation}
where the expectation $E$ is taken with respect to  the stationary distribution corresponding to the true parameter $\theta^*$. Thus, $T_{s,a,s'}$ is the expectation of the counter of the transition ${s,a,s'}$ divided by the total number of transitions (or the stationary distribution $P(s,a,s')$).
\begin{remark}
Note that although $\hat p(s,s')$ in~\eqref{eq:state_transition_est} is known, it represents only the transition probability under the fixed behavior policy. Therefore, we still need to estimate the state transition model $T_{sas'}$ for the state predictor in \eqref{eq: update u with the other policy}.
\end{remark}

The proposed recursive estimation of $T_{s,a,s'}$ is given by
{\small
\begin{equation}\label{eq: transition_estimation}
    \begin{aligned}
        \begin{bmatrix}
            T_{n+1}(1, a_n, 1)\\
            T_{n+1}(1, a_n, 2)\\
            \vdots     \\
            T_{n+1}(I, a_n, I)
        \end{bmatrix} &=
        \begin{bmatrix}
            T_{n}(1, a_n, 1)\\
            T_{n}(1, a_n, 2)\\
            \vdots     \\
            T_{n}(I, a_n, I)
        \end{bmatrix}\\
        &+
        \epsilon_n
        \begin{bmatrix}
            \hat{p}_n(1,1)(1 -  T_n(1,a_n,1)) \\
            \hat{p}_n(1,2)(1 -  T_n(1,a_n,2)) \\
            \vdots \\
            \hat{p}_n(I,I)(1 - T_n(I,a_n,I))
        \end{bmatrix}.
    \end{aligned}
\end{equation}
}
We note that the estimation in~\eqref{eq: transition_estimation} uses $\hat{p}(s,s')$ as a surrogate for $P(s'|s,a)$ in \eqref{eq:Bellam_Optimality}. The ODE corresponding to~\eqref{eq: transition_estimation} is
{\small
\begin{equation*}
        \begin{bmatrix}
            \dot{T}_{1,a,1}\\
            \dot{T}_{1,a,2}\\
            \vdots     \\
            \dot{T}_{I,a,I}
        \end{bmatrix}
        =
        \frac{1}{\bar{u}_a}
        \begin{bmatrix}
        \bar{p}(1,a,1)(1 - T_{1,a,1}) \\
        \bar{p}(1,a,2)(1 - T_{1,a,2}) \\
        \vdots \\
        \bar{p}(I,a,I)(1 - T_{I,a,I})
        \end{bmatrix}
        ,\quad a\in {\cal A}.
\end{equation*}
}
Following the same procedure in the proof of {\it Proposition~\ref{proposition: convergence of Q estimation}}, we can show that $t_n(s,a,s')$ converges to $\bar{p}(s,a,s')$, where $\bar{p}(s,a,s')$ denotes the marginal distribution of the transition from $s$ to $s'$ after taking $a$ with  respect to the invariant distribution of the entire process. Since we estimate the joint distribution, the conditional distribution $T_a(s,s')$ can be calculated by dividing the joint probabilities with marginal probabilities.

%%%%%%%%%%%%%%%%%%%%%%%%%%%%%%%%%%%%%%%%%%%%%%%%%%%%%%%%%%%%%%%%%%%%%%%%%%%%%%%%%
\section{A Numerical Example}
In this simulation, we implement the HMM Q-learning for a finite state POMDP example, where 4 hidden states are observed through 2 observations with the discount factor $\gamma = 0.95$ as specified below:
{\small
\begin{equation*}
\mathbf{T}=\left[
\begin{bmatrix}
.6& .2& .1& .1\\
.2& .1& .6& .1\\
.1&  .1& .1& .7\\
.4& .1& .1& .4
\end{bmatrix},
\begin{bmatrix}
.1& .2& .2& .5\\
.1& .6& .1& .2\\
.1& .2& .6& .1\\
.1& .1& .2& .6
\end{bmatrix}
\right],
\end{equation*}
\begin{equation*}
\mathbf{O}=
\begin{bmatrix}
.95& .05\\
.95& .05\\
.05& .95\\
.05& .95
\end{bmatrix},
\;
\mathbf{R}=
\begin{bmatrix}
0& 0.& -20.& +20.\\
0& 0.& +20.& -20.\\
\end{bmatrix},
\;
\sigma = 1.
\end{equation*}
}
The following behavior policy $\mu(o)$ is used to estimate the HMM, the transition model, and the Q-function
{\small
\begin{equation*}
\mu =
\begin{bmatrix}
.6& .4\\
.3& .7
\end{bmatrix},
\quad
\mu_{i,j} = P(a=j|o=i).
\end{equation*}
}
The diminishing step size is chosen as $\epsilon_n = n^{-0.4}$ for $n\geq1$.

\subsection{Estimation of the HMM and Q-function}
Figure~\ref{fig: log-likelihood} shows that the mean of the sample conditional log-likelihood density $\log \mathbf{b}^\top(y_n;\theta_n)\mathbf{u}_n$ increases. Figure~\ref{fig: sigma} shows that $\sigma_n$ converges to the true parameter $\sigma^*=1.0$.
\begin{figure}
\centering
\begin{subfigure}{.245\textwidth}
  \centering
  \includegraphics[width=1.\linewidth]{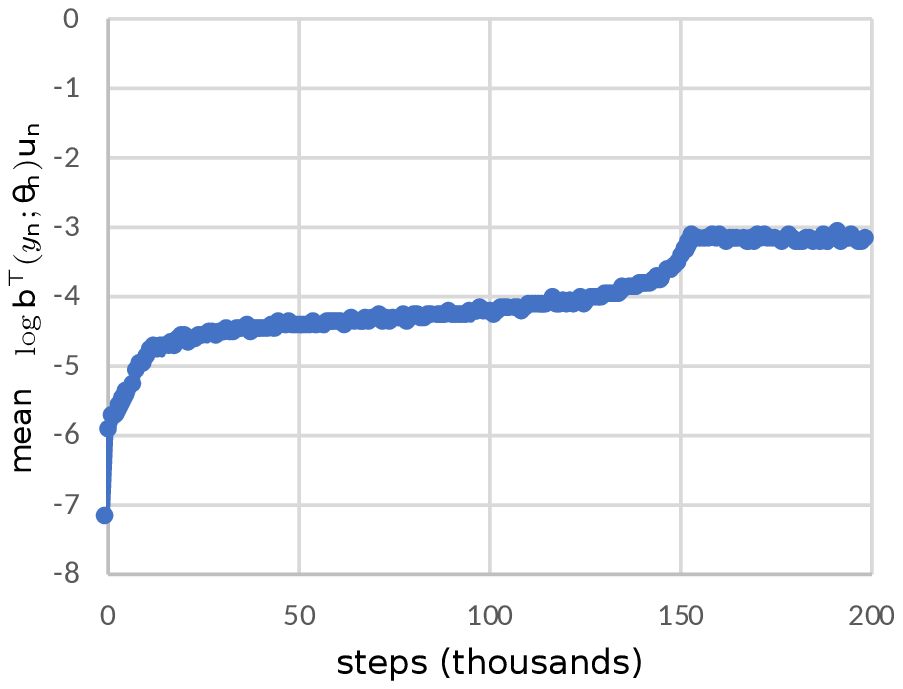}
  \caption{$\log \mathbf{b}^\top(y_n;\theta_n)\mathbf{u}_n$.}
  \label{fig: log-likelihood}
\end{subfigure}%
\begin{subfigure}{.245\textwidth}
  \centering
  \includegraphics[width=1.\linewidth]{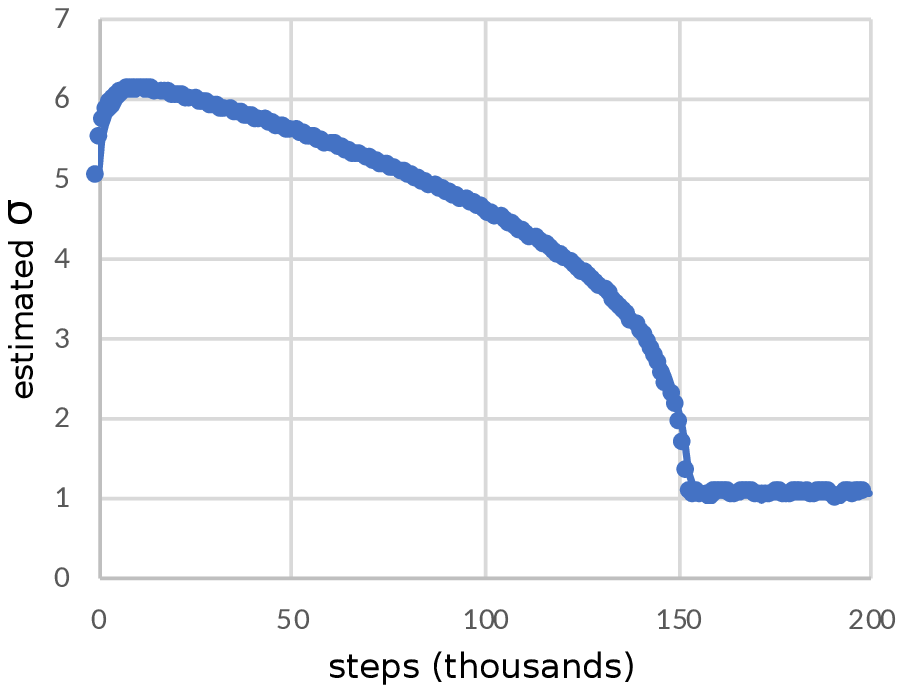}
  \caption{$\sigma(\theta_n)$}
  \label{fig: sigma}
\end{subfigure}
\caption{The mean of the sampled conditional likelihood $\log \mathbf{b}^\top(y_n;\theta_n)\mathbf{u}_n$ increases as the estimated $\sigma(\theta_n)$ converges to the true $\sigma = 1$.}
\label{fig: likelihood_and_sigma}
\end{figure}

To validate the estimation of the Q-function in~\eqref{eq: Q_estimator_with_HMM}, we run three estimations of Q-function in parallel: (i) Q-learning~\cite{watkins1992q} with full state observation $s$, (ii) Q-learning with partial observation $o$, (iii) HMM Q-learning. Figure~\ref{fig: maxQ} shows  $\max_{s,a} Q_n(s, a)$ for all three algorithms.
\begin{figure}[thpb]
\centering
 \includegraphics[width=0.3\textwidth]{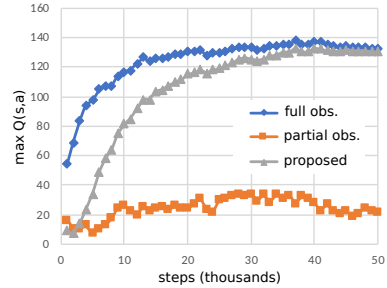}
 \caption{$\max_{s,a} Q_n(s, a)$ is greater with full observation than partial observation. The proposed HMM Q-learning's $\max_{s,a} Q_n(s, a)$ converges to the fully observing Q-learning's.}
\medskip
\label{fig: maxQ}
\end{figure}

After 200,000 steps, the iterates of $Q^\text{full}_n$,  $Q^\text{partial}_n$ and $Q^\text{hmm}_n$ at $n=2 \times 10^5$ are as follows:
{\small
\begin{equation*}
\begin{aligned}
Q^\text{full}_n &=
\begin{bmatrix}
107.4& 103.4&  99.3& 133.8\\
114.7& 107.6& 102.4&  98.0
\end{bmatrix}^\top,\\
Q^\text{partial}_n &=
\begin{bmatrix}
20.1& 21.6\\
18.9&  9.1
\end{bmatrix}^\top, \\
Q^\text{hmm}_n &=
\begin{bmatrix}
133.0& 106.0& 105.9&  99.1\\
 98.1& 111.2& 111.7& 105.4
\end{bmatrix}^\top,
\end{aligned}
\end{equation*}
}
where the $(i,j)$ elements of the $Q$ matrices are the estimates of the Q-function value, when $a = i, s = j$. Similar to the other HMM estimations (from unsupervised learning task), the labels of the inferred hidden state do not match the labels assigned to the true states. Permuting the state indices $\{1,2,3,4\}$ to $(2,3,4,1)$ in order to have better matching between the estimated and true Q-function, we compare the estimated Q-function as follows:
{\small
\begin{equation*}
\begin{aligned}
Q^\text{permuted}_n &=
\begin{bmatrix}
106.0& 105.9&  99.1& 133.0\\
111.2& 111.7& 105.4&   98.1
\end{bmatrix}^\top,\\
Q^\text{full}_n &=
\begin{bmatrix}
107.4& 103.4&  99.3& 133.8\\
114.7& 107.6& 102.4&  98.0
\end{bmatrix}^\top.
\end{aligned}
\end{equation*}
}
This permutation is consistent with the estimated observation $\mathbf{O}(\theta_n)$ as below:
{\small
\begin{equation*}
\mathbf{O}(\theta_n) =
\begin{bmatrix}
.066& .934\\
.943& .057\\
.947& .053\\
.052& .948
\end{bmatrix},
\quad
\mathbf{O}(\theta^*)=
\begin{bmatrix}
.950& .050\\
.950& .050\\
.050& .950\\
.050& .950
\end{bmatrix}.
\end{equation*}
}

\subsection{Dynamic Policy with Partial Observations}
When the model parameters of POMDP are given, the Bayesian state belief filter can be used to make decisions based on the state belief. The use of the Bayesian state belief filter has demonstrated improved performance as compared to the performance of the standard RL algorithms with partial observation~\cite{littman1995learning, karkus2017qmdp}.

After a certain stopping criterion is satisfied, we fix the parameter. The fixed POMDP parameters $(\mathbf{T}_{\theta_l}, \mathbf{O}_{\theta_l}, \mathbf{R}_{\theta_l}, \sigma_{\theta_l})$ are used in the following Bayesian state belief filter
{\small
\begin{equation}\label{eq: update u with the other policy}
     \mathbf{u}_{n+1}=\frac{\mathbf{T}^\top_{\theta_l}(a_n) \mathbf{B}(y_n;\theta_l)\mathbf{u}_{n}}{\mathbf{b}^\top(y_n;\theta_l) \mathbf{u}_{n}},
\end{equation}
}
where $\mathbf{u}_n = [u_{n,1}, u_{n,2}, \dots, u_{n,I}]^\top,$ and $u_{n,i} = P(s_n=i|y_0, y_1, \dots, y_n; \theta_l)$.

The action $a^*$ is chosen based on the expectation of the Q-function on the state belief distribution and the current observation $o_n$
{\small
\begin{equation}\label{eq: dynamic policy}
a^* = \argmax_a \sum_{i}^I Q_{\theta_l}(s=i, a)P(s_n=i|o_n, \mathbf{u}_n; \theta_l),
\end{equation}
}
where
{\small
\begin{equation*}
P(s_n=i|o_n, \mathbf{u}_n; \theta_l) = \frac{P(o_n|s_n=i;\theta_l) u_{n,i}}{\sum_{j}^I P(o_n|s_n=j;\theta_l) u_{n,j}}.
\end{equation*}
}
\begin{remark}
Similar to output feedback control with state observer, the policy in~\eqref{eq: dynamic policy} uses a state predictor to choose an action.
\end{remark}

We tested the dynamic policy consisting of~\eqref{eq: update u with the other policy} and~\eqref{eq: dynamic policy} at every thousand steps of the parameter estimation. Each test comprises 100 episodes of running the POMDP with the policy. Each episode in the test takes 500 steps. Then the mean rewards of total $100\times500$ steps are marked and compared with the policies of the Q-learning with full state observation and partial state observation~\cite{watkins1992q}. Figure~\ref{fig: mean reward} shows that the proposed HMM Q-learning performs better than the Q-learning with partial observation.
\begin{figure}[thpb]
\centering
 \includegraphics[width=0.3\textwidth]{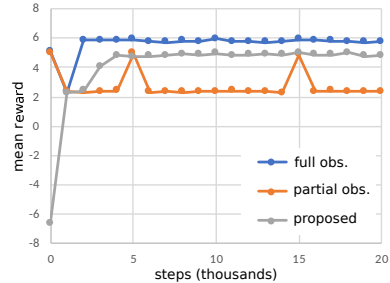}
 \caption{mean rewards from Q-learning with full observation, Q-learning with partial observation, and the proposed HMM Q-learning.}
\medskip
\label{fig: mean reward}
\end{figure}

\section{Conclusion}
We presented a model-based approach to the problem of reinforcement learning with incomplete observation. Since the controlled POMDP is an HMM, we invoked results from Hidden Markov Model (HMM) estimation. Based on the convergence of the HMM estimator, the optimal action-value function $Q^*(s, a)$ is learned despite the hidden states. The proposed algorithm is recursive, i.e. only the current sample is used so that there is no need for replay buffer, in contrast to the other algorithms for POMDP~\cite{hausknecht2015deep, heess2015memory}.

We proved the convergence of the recursive estimator using the ergodicity of the underlying Markov chain for the HMM estimation~\cite{legland1997recursive, krishnamurthy2002recursive}. The approach developed in stochastic approximation~\cite{kushner2003stochastic} is used to show the convergence of the estimators in spite of correlated data samples and asynchronous update. Also, we presented a numerical example where the simulation shows the convergent behavior of the recursive estimator.
%
%%\addtolength{\textheight}{-1cm}   % This command serves to balance the column lengths
%                                  % on the last page of the document manually. It shortens
%                                  % the textheight of the last page by a suitable amount.
%                                  % This command does not take effect until the next page
%                                  % so it should come on the page before the last. Make
%                                  % sure that you do not shorten the textheight too much.
%
%%%%%%%%%%%%%%%%%%%%%%%%%%%%%%%%%%%%%%%%%%%%%%%%%%%%%%%%%%%%%%%%%%%%%%%%%%%%%%%%%
%
%
%
%%%%%%%%%%%%%%%%%%%%%%%%%%%%%%%%%%%%%%%%%%%%%%%%%%%%%%%%%%%%%%%%%%%%%%%%
%
%
%%%%%%%%%%%%%%%%%%%%%%%%%%%%%%%%%%%%%%%%%%%%%%%%%%%%%%%%%%%%%%%%%%%%%%%%%%%%%%%%%
\bibliographystyle{IEEEtran}
\bibliography{mybibfile}

\section*{Appendix}
\subsection{Convergence of the HMM estimation}\label{Appendix A.}
The convergence result in~\cite{krishnamurthy2002recursive} is briefly stated first. Then we verify that the assumptions ({\it\textbf{C}~\ref{assumption: aperiodic-irreducible-P}, \textbf{C}~\ref{assumption: differentiable_model}, \textbf{C}~\ref{assumption: geometrically_ergodic_extended_MC}, \textbf{C}~\ref{assumption: regularity_on_the_update_function}}) from~\cite{krishnamurthy2002recursive} are satisfied for the HMM, which is the POMDP on finite state-action set excited by the behavior policy.

The assumptions for the convergence of the HMM estimator are given as follows:
%================================================================
\begin{hmm_assumption}\label{assumption: aperiodic-irreducible-P}
The transition matrix $\mathbf{P}_{\theta^*}$ of the true parameter $\theta^*$ is aperiodic and irreducible.
\end{hmm_assumption}
%================================================================
\begin{hmm_assumption}\label{assumption: differentiable_model}
The mapping for the transition matrix $\theta \rightarrow \mathbf{P}_{\theta}$ is twice differentiable with bounded first and second derivatives and Lipschitz continuous second derivative. Furthermore, for any $y_n$, the mapping $\theta \rightarrow \mathbf{b}(y_n;\theta)$ is three times differentiable; $\mathbf{b}(y_n;\theta)$ is continuous on $\mathcal{Y}:=\mathcal{O}\times\mathbb{R}\times\mathcal{A}$ for each $\theta \in \Theta$.
\end{hmm_assumption}
%================================================================
\begin{hmm_assumption}\label{assumption: geometrically_ergodic_extended_MC}
Under the probability measure corresponding to the true parameter $\mathbf{\theta^*}$, the extended Markov chain\footnote{The update laws in~\eqref{eq: update-u},~\eqref{eq: update-omega} determine the next $\mathbf{u}_{n+1}$ and $\mathbf{\omega}_{n+1}$ only using the current $\mathbf{u}_{n}$ and $\mathbf{\omega}_{n}$, so the extended chain is still Markov.}
\begin{equation*}
    \{ s_n, y_n, \mathbf{u}_n, \mathbf{\omega}_n \}
\end{equation*}
associated with $\theta \in \Theta$ is geometrically ergodic.
\end{hmm_assumption}
%================================================================

The ordinary differential equation (ODE) approach~\cite{kushner2003stochastic} for the stochastic approximation is used to prove the convergence. Rewrite~\eqref{eq: HMM_estimator} as
\begin{equation}\label{stochastic-approx-HMM}
    \theta_{n+1} = \theta_n + \epsilon_n \mathbf{S}\left(y_n, \mathbf{u}_n, \mathbf{\omega}_n ;\theta_n\right) + \epsilon_n M_n,
\end{equation}
where $M_n$ is the projection term, i.e. it is the vector of shortest Euclidean length needed to bring $\theta_n + \epsilon_n \mathbf{S}(y_n, \mathbf{u}_n, \mathbf{\omega}_n;\theta_n)$ back to the constraint set $H$, if it escapes from $H$.  The ODE approach shows that the piecewise constant interpolation over continuous time  converges to the ODE, which has an invariant set with desirable property. In our problem, the set with maximum likelihood is desired. For technical details on the ODE approaches, we refer to~\cite{kushner2003stochastic}.

Define a piece-wise constant interpolation of $\theta_n$ as follows:
\begin{equation*}
    t_n = \sum_{i=0}^{n-1} \epsilon_i, \quad t_0=0,
\end{equation*}
\begin{equation*}
    m(t)=\begin{cases}
        n; \; t_n \leq t < t_{n+1} & \text{for } \; t \geq 0\\
        0 &\text{for} \; t < 0.
    \end{cases}
\end{equation*}
Define the piece-wise constant process $\theta^0(t)$ as:
\begin{equation*}
    \theta^0(t)=\begin{cases}
        \theta_0, & \text{for } t \leq 0\\
        \theta_n, & \text{for } t_n \leq t  < t_{n+1}, \quad \text{for}\; t\geq 0.
    \end{cases}
\end{equation*}
Define the shifted sequence $\theta^{n}(\cdot)$ to analyze the asymptotic behavior:
\begin{equation*}
    \theta^n(t) = \theta^0(t_n + t), \quad \text{for} \; t\in(-\infty, \infty).
\end{equation*}
Similarly, define $M^0(\cdot)$ and $M^n(\cdot)$ by
\begin{equation*}
    M^0(t) =
    \begin{cases}
        \sum_{i=0}^{m(t)-1}\epsilon_i M_i, &\text{for}\; t\geq0\\
        0, &\text{for} \; t<0,
    \end{cases}
\end{equation*}
and
\begin{equation*}
    M^n(t) =
    \begin{cases}
        M^0(t_n+t) - M^0(t), &\text{for} \; t\geq 0\\
        -\sum_{i=m(t_n+t)}^{n-1} \epsilon_i M_i, &\text{for} \; t<0.
    \end{cases}
\end{equation*}
The ODE approach aims to show the convergence of the piece-wise constant interpolation to the following projected ODE:
\begin{equation}\label{eq:ODE}
    \dot{\theta} = \mathbf{H}(\theta) + \tilde{m}, \quad \theta(0) = \theta_0,
\end{equation}
where $\mathbf{H}(\theta) = E\mathbf{S}(y_n, \mathbf{u}_n, \mathbf{\omega}_n;\theta)$, and $\tilde{m}(\cdot)$ is the projection term to keep $\theta$ in $H$. Here, the expectation $E$ is taken with respect to  the stationary distribution corresponding to the true parameter $\theta^*$.
Define the following set of points along the trajectories:
\begin{equation*}
    L_{H} = \{ \theta;\; \theta \; \text{be a limit point of \eqref{eq:ODE}}, \; \theta_0 \in H\},
\end{equation*}
\begin{equation*}
    \hat{L}_{H}=\{\theta\in G_1;\; \mathbf{H}(\theta)+\tilde{m}=0\},
\end{equation*}
\begin{equation*}
    L_{ML} = \{\argmax \, l(\theta)\},
\end{equation*}
where $l(\theta)$ is the likelihood calculated with respect to the  stationary distribution corresponding to the true parameter $\theta^*$.

\begin{hmm_assumption}[see A2 in~\cite{krishnamurthy2002recursive}]\label{assumption: regularity_on_the_update_function}
For each $\theta \in \Theta$, $\{\mathbf{S}(y_n, \mathbf{u}_n, \mathbf{\omega}_n;\theta)\}$ is uniformly integrable, $E[\mathbf{S}(y_n, \mathbf{u}_n, \mathbf{\omega}_n;\theta)]= \mathbf{H}(\theta)$, $\mathbf{H}(\cdot)$ is continuous, and $\mathbf{S}(y_n, \mathbf{u}_n, \mathbf{\omega}_n;\theta)$ is continuous for each $(y_n, \mathbf{u}_n, \mathbf{\omega}_n)$. There exist nonnegative measurable functions $\tilde{\rho}(\cdot)$ and $\hat{\rho}(\cdot)$, such that $\tilde{\rho}(\cdot)$ is bounded on bounded $\theta$ set, and
\begin{equation*}
    |\mathbf{S}(y_n, \mathbf{u}_n, \mathbf{\omega}_n;\theta) - \mathbf{S}(y_n, \mathbf{u}_n, \mathbf{\omega}_n;\phi)| \geq \tilde{\rho}(\theta - \phi)\hat{\rho}(y_n, \mathbf{u}_n, \mathbf{\omega}_n),
\end{equation*}
such that $\tilde{\rho}(\theta) \rightarrow 0$ as $\phi \rightarrow 0$, and
{\small
\begin{equation*}
    P\left( \limsup_n \sum_{i=n}^{m(t_n+s)}\epsilon_i \hat{\rho}(y_i, \mathbf{u}_i, \mathbf{\omega}_i) < \infty \right) = 1, \, \text{ for some } s>0.
\end{equation*}
}
\end{hmm_assumption}
%\begin{assumption}\label{assumption:nice_limit_set}
%Suppose that $L^1_G$ is a subset of $L_G$ and $L_{ML}$ is locally asymptotically stable. For any initial condition $\theta_0 \notin L^1_G$, the trajectories of the ODE in~\eqref{eq:ODE} goes to $L_{ML}$.
%\end{assumption}
\vspace{5mm}
\begin{theorem}[see Theorem 3.4 in~\cite{krishnamurthy2002recursive}]\label{thm: hmm_estimator} Assume {\it \textbf{C}~\ref{assumption: aperiodic-irreducible-P}, \textbf{C}~\ref{assumption: differentiable_model}, \textbf{C}~\ref{assumption: geometrically_ergodic_extended_MC}, and \textbf{C}~\ref{assumption: regularity_on_the_update_function}} hold. There is a null set $\tilde{N}$, such that for all $\omega\notin \tilde{N}$, $\{\theta^n(\omega, \cdot), M^n(\omega, \cdot)\}$ is equicontinuous (in the extended sense as in~\cite[p.~102]{kushner2003stochastic}). Let $(\theta(\omega, \cdot), M(\omega, \cdot))$ denote the limit of some convergent subsequence. Then the pair satisfies the projected ODE~\eqref{eq:ODE}, and ${\theta_n}$ converges to an invariant set of the ODE in $H$.
%Further suppose Assumption~\ref{assumption:nice_limit_set}. Then the limit points are in $L^1_{G_1} \cup $
\end{theorem}

We verify that the assumptions in Theorem~\ref{thm: hmm_estimator} are satisfied with the HMM.  First, we make an assumption on the behavior policy.
% \begin{assumption}\label{assumption:nice_behavioral_policy}
% The transition probability matrix $\mathbf{P}$ determined by the transition $\mathbf{T}$, the observation $\mathbf{O}$, and the behavior policy $\mu(o)$ is aperiodic and irreducible. Furthermore, the behavior policy has strictly positive probability to choose the action for all action choices.
% \end{assumption}

% \begin{assumption}\label{assumption:positive observation probability}
% The observation probability matrix $\mathbf{O}$ is positive, i.e. $\mathbf{O}_{i,j} > 0$.
% \end{assumption}

{\it \textbf{Assumption}~\ref{assumption:nice_behavioral_policy}} is sufficient for {\it\textbf{C}~\ref{assumption: aperiodic-irreducible-P}}.

We verify {\it\textbf{C}~\ref{assumption: differentiable_model}} as follows. The first part of the assumption depends on the parametrization of the transition $\mathbf{P}_\theta$. The exponential parametrization (or called Softmax function) for $\mathbf{P}_\theta$ is a smooth function of the parameter $\theta$. So, $\mathbf{P}_\theta$ is twice differentiable with bounded first and second derivatives and Lipschitz continuous second derivative. For the HMM model in this paper, $\mathbf{b}(y_n;\theta)$ defined in~\eqref{eq: Output-likelihood} is a vector of density functions of normal distribution multiplied by conditional probabilities, i.e. $b_i(y_n) = P(o_n|s_n=i; \theta)P(a_n|s_n=i; \theta)p(r_n,|s_n=i, a_n; \theta)$. Since the density model is given by normal distribution, it is easy to see that $\mathbf{b}(y_n;\theta)$  is three times differentiable, and the $\mathbf{b}(y_n;\theta)$ is continuous on $\mathcal{O}\times\mathbb{R}\times\mathcal{A}$ with Euclidean metric.

{\it \textbf{C}~\ref{assumption: geometrically_ergodic_extended_MC}} states the geometric ergodicity of the extended Markov chain $\{ s_n, y_n, \hat{p}_n, \omega_n \}$. A sufficient condition for the ergodicity of the extended Markov chain is that  {\it\textbf{C}~\ref{assumption: aperiodic-irreducible-P}} holds, and the following $\Delta_2^{(0)}, \Delta_4^{(0)}$ are finite~(see Remark 2.6 in~\cite{krishnamurthy2002recursive}):
{\small
\begin{equation}\label{eq: integration_bounds}
\begin{aligned}
    \delta^{(s)}(y)      &= \sup_{\theta\in\Theta} \; \max_{k_1, \dots, k_s \in \{1, ..., L\}} \frac{\max_{i\in\mathcal{S}}|\partial^s_{k_1, \dots, k_s} b_i(y;\theta)|}{\min_{j\in\mathcal{S} }b_j(y;\theta)}, \\
    \Delta^{(s)}_\iota &= \sup_{\theta\in\Theta} \max_{i\in \mathcal{S}}\int_{\mathcal{Y}} \left[\delta^{(s)}(y)\right]^\iota b_i(y;\theta) dy, \\
    \Gamma_\iota      &= \sup_{\theta\in\Theta} \max_{i\in \mathcal{S}}\int_{\mathcal{Y}} \left[\max_{j\in\mathcal{S}}|\log b_j(y;\theta)|\right]^\iota b_i(y;\theta^*) dy, \\
    \bar{Y}_\iota        &= \sup_{\theta\in\Theta} \max_{i\in \mathcal{S}}\int_{\mathcal{Y}} |r|^\iota b_i(y;\theta) dy.
\end{aligned}
\end{equation}
}
To this end, we compute the bound on $\Delta_2^{(0)}, \Delta_4^{(0)}$ in the following lemma.
\begin{lemma}\label{lemma: first bound}
$\Delta_2^{(0)}$ and $\Delta_4^{(0)}$ are finite.
\end{lemma}
\begin{proof}
We need to show that the following expressions are bounded:
{\small
\begin{equation}
\begin{aligned}
    \delta^{(0)}(y)      &= \sup_{\theta\in\Theta} \; \frac{\max_{i\in\mathcal{S}}b_i(y;\theta)}{\min_{j\in\mathcal{S} }b_j(y;\theta)}, \\
    \Delta^{(0)}_2 &= \sup_{\theta\in\Theta} \max_{i\in \mathcal{S}}\int_{\mathcal{Y}} \left[\delta^{(s)}(y)\right]^2 b_i(y;\theta) dy, \\
    \Delta^{(0)}_4 &= \sup_{\theta\in\Theta} \max_{i\in \mathcal{S}}\int_{\mathcal{Y}} \left[\delta^{(s)}(y)\right]^4 b_i(y;\theta) dy,
\end{aligned}
\end{equation}
}
for the $b_i(y;\theta)$ given by
{\small
\begin{align*}
b_i(y;\theta) &= p(y|s=i ;\theta) \\
              &= P(o|s=i; \theta)P(a|s=i; \theta)p(r,|s=i, a; \theta) \\
              &= \mathbf{O}_{i,o} \, \mu(o)\frac{1}{\sqrt{2\pi\sigma_\theta^2}}\exp\left({-\frac{(r - \mathbf{R}_{a, i})^2}{2 \sigma_\theta^2}}\right),
\end{align*}
}
where $\mathbf{O}_{i,o}:=\frac{\exp(o_{i,o})}{\sum_{j'=1}^{J} \exp(o_{i,j'})}$, $o_{i,j}$ is the $(i,j)$\textsuperscript{th} element of $\mathbf{O}_\theta$, and $\mathbf{R}_{a,i}$ is the $(a,i)$\textsuperscript{th} element of $\mathbf{R}_\theta$.

The following bounds hold for some $\gamma_0, \gamma_1, \gamma_2 > 0$, since the elements in the probability matrix $\mathbf{O}_\theta$ are strictly positive, and the values of  $\mathbf{R}_\theta$ verify
{\small
\begin{align*}
    &\frac{b_i(y;\theta)}{b_j(y;\theta)}\\
    &=\frac{\mathbf{O}_{i,o}}{\mathbf{O}_{j,o}}\exp\left(\frac{-(r - \mathbf{R}_{a, i})^2 + (r - \mathbf{R}_{a, j})^2}{2\sigma_\theta^2}\right)\\
    %&\leq
    %\frac{1}{\min_{j'}\mathbf{O}_{j',o}}\exp\left({\frac{r(\mathbf{R}_{a, i} - \mathbf{R}_{a, j} ) - \frac{1}{2}(\mathbf{R}_{a, i} - \mathbf{R}_{a, j})(\mathbf{R}_{a, i} - \mathbf{R}_{a, j})}{2\sigma_\theta^2}}\right) \\
    &\leq
    \frac{1}{\min_{j'}\mathbf{O}_{j',o}}\exp\left({\frac{\max_{i,j}|\mathbf{R}_{a, i} - \mathbf{R}_{a, j}|\left(|r| + \max_{i'}\mathbf{R}_{a, i'}\right)}{2\sigma_\theta^2}}\right) \\
    &\leq
    \gamma_0 \exp(\gamma_1 |r| + \gamma_2).
\end{align*}
}
Hence, $\delta^{(0)}(y)<\infty$ for a fixed $y = (o,r,a)$.

Calculating $\Delta^{(0)}_\iota$ for $\iota\geq1$, we have
{\small
\begin{equation*}
\begin{aligned}
\Delta^{(0)}_\iota
&= \sup_{\theta\in\Theta} \max_{i\in \mathcal{S}}\int_{\mathcal{Y}} \left[\delta^{(s)}(y)\right]^\iota b_i(y;\theta) dy \\
&\leq \sup_{\theta\in\Theta} \max_{i, a}\int_{\mathbb{R}} \gamma_0^\iota \exp(\iota\gamma_1 |r| + \iota\gamma_2) \left(\frac{\exp\left({-\frac{(r - \mathbf{R}_{a, i})^2}{2 \sigma_\theta^2}}\right)}{\sqrt{2\pi\sigma_\theta^2}}\right) dr \\
& \leq \sup_{\theta\in\Theta} \max_{i, a}\int_{-\infty}^0 \gamma_3\exp(-\gamma_4(r - \lambda_{i,a})^2)dr \\
 &\quad + \sup_{\theta'\in\Theta} \max_{i', a'}\int_{0}^{+\infty} \gamma_3\exp(-\gamma_5(r - \lambda_{i',a'})^2)dr,
\end{aligned}
\end{equation*}
}
where $\gamma_3, \gamma_4, \gamma_5>0$ and  $\lambda_{i,a}$ are calculated by simplifying the terms.
For all $\theta\in\Theta$, $(i,a)\in\mathcal{S}\times\mathcal{A}$, we have
\begin{equation*}
\int_\mathbb{R} \gamma_3\exp(-\gamma_4(r - \lambda_{i,a})^2)dr < \infty,
\end{equation*}
since  the integrand is given in the  form of normal distribution. Hence $\Delta^{(0)}_\iota < \infty$ for $\iota\geq1$.
\end{proof}

To verify uniform integrability and Lipschitz continuity in \it{\textbf{C}~\ref{assumption: regularity_on_the_update_function}}, a sufficient condition is that $\Delta_\iota^{(1)}$, $\Gamma_2$, and $\bar{Y}_2$ are finite for all $\iota\geq0$ (see Remark 3.1 in~\cite{krishnamurthy2002recursive}). Next lemma proves that result.
\begin{lemma}\label{lemma: second bound}
$\Delta_2^{(1)}$, $\Gamma_2$, and $\bar{Y}_2$ are finite.
\end{lemma}
\begin{proof}
First, we need to show that $\Delta_2^{(1)}$, given by
{\small
\begin{equation*}
\Delta^{(1)}_2 = \sup_{\theta\in\Theta} \max_{i\in \mathcal{S}}\int_{\mathcal{Y}} \left[\delta^{(1)}(y)\right]^2 b_i(y;\theta) dy,
\end{equation*}
where
\begin{equation*}
\delta^{(1)}(y) = \sup_{\theta\in\Theta} \; \max_{l \in \{1, ..., L\}} \frac{\max_{i\in\mathcal{S}}|\partial_{\theta^{(l)}} b_i(y;\theta)|}{\min_{j\in\mathcal{S} }b_j(y;\theta)}
\end{equation*}
} is bounded.
Calculating $\frac{|\partial_{\theta^{(l)}} b_i(y;\theta)|}{b_j(y;\theta)}$ for each $\theta^{(l)}\in\{o_{i,j}, \mathbf{R}_{a,i}, \sigma_\theta \}$, we have:
{\small
\begin{equation*}
\begin{aligned}
    \frac{|\partial_{\mathbf{o}_{i,j}} b_i(y;\theta)|}{b_j(y;\theta)}&=\begin{cases}
        (1-\mathbf{O}_{i,j}) \frac{|b_i(y;\theta)|}{b_j(y;\theta)}, & \text{for } j = o \\
        \mathbf{O}_{i,j} \frac{|b_i(y;\theta)|}{b_j(y;\theta)}, & \text{for } j \neq o,
    \end{cases}\\
\frac{|\partial_{\mathbf{R}_{a,i}} b_i(y;\theta)|}{b_j(y;\theta)} &=\frac{(r - \mathbf{R}_{a, i})}{\sigma_\theta^2}\frac{|b_i(y;\theta)|}{b_j(y;\theta)}, \\
\frac{|\partial_{\sigma_\theta} b_i(y;\theta)|}{b_j(y;\theta)} &=-\left(\frac{2(r - \mathbf{R}_{a, i})^2+\sigma_\theta^2}{\sigma_\theta^3}\right)\frac{|b_i(y;\theta)|}{b_j(y;\theta)}.
\end{aligned}
\end{equation*}
}
In the proof of Lemma~\ref{lemma: first bound}, we showed that $\frac{|b_i(y;\theta)|}{b_j(y;\theta)}\leq\gamma_0 \exp(\gamma_1 |r| + \gamma_2)$. Using integration by parts, it is easy to verify that $\int_\mathbb{R}r^\iota\exp(-r^2)dr < \infty$ for $\iota\in\{1,2,\dots\}$. Using the calculated bounds, it is straightforward to show that $\Delta^{(1)}_2<\infty$.

Secondly, we need to show that $\Gamma_2$, given by
\begin{equation*}
\Gamma_2 = \sup_{\theta\in\Theta} \max_{i\in \mathcal{S}}\int_{\mathcal{Y}} \left[\max_{j\in\mathcal{S}}|\log b_j(y;\theta)|\right]^2 b_i(y;\theta^*) dy
\end{equation*}
is bounded. Indeed, its boundedness follows from the fact that
\begin{equation*}
|\log b_j(y;\theta)| \leq (r - \mathbf{R}_{a, i})^2 + \gamma
\end{equation*}
holds for some constant $\gamma > 0$, and $\int_\mathbb{R}r^\iota\exp(-r^2)dr < \infty$ for $\iota\in\{1,2,\dots\}$, $\Gamma_2 < \infty$.

Lastly, $\bar{Y}_2$, given by
\begin{equation*}
\bar{Y}_2 = \sup_{\theta\in\Theta} \max_{i\in \mathcal{S}}\int_{\mathcal{Y}} |r|^2 b_i(y;\theta) dy
\end{equation*}
 is bounded, since $\int_\mathbb{R}r^\iota\exp(-r^2)dr < \infty$ for $\iota\in\{1,2,\dots\}$.
\end{proof}

Now, we have verified {\it\textbf{C}~\ref{assumption: aperiodic-irreducible-P}, \textbf{C}~\ref{assumption: differentiable_model}, \textbf{C}~\ref{assumption: geometrically_ergodic_extended_MC},} and {\it\textbf{C}~\ref{assumption: regularity_on_the_update_function}} for {\it\textbf{Theorem}~\ref{thm: hmm_estimator}}, which states the convergence of $\theta_n$ to an invariant set.
\QEDA

\subsection{Convergence of the Q-function Estimation with the HMM State Predictor}\label{Appendix B.}
We invoke the convergence result for asynchronous update stochastic approximation algorithm from~\cite{kushner2003stochastic}. 

\subsubsection{Preliminaries}
For $\alpha = 1, \dots, r$, let
\begin{equation*}
    \theta^\epsilon_{n+1, \alpha} = \Pi_{[a_\alpha, b_\alpha]}\left[ \theta^\epsilon_{n, \alpha} + \epsilon Y^\epsilon_{n,\alpha} \right]
    = \theta^\epsilon_{n,\alpha} + \epsilon Y^\epsilon_{n,\alpha} + \epsilon Z^\epsilon_{n, \alpha}
\end{equation*}
define the \emph{scaled interpolated real-time} $\tau^\epsilon_{n,\alpha}$ as follows:
\begin{equation*}
    \tau^\epsilon_{n,\alpha} = \epsilon \sum_{i=0}^{n-1} \delta \tau^{\epsilon}_{i,\alpha},
\end{equation*}
where $\delta \tau^\epsilon_{n, \alpha}$ denotes the real-time between the $n^{\rm th}$ and the $(n+1)^{\rm th}$ update of the $\alpha^{\rm th}$ component of $\theta$. Let $\theta^\epsilon_\alpha(\cdot)$ denote the interpolation of $\{\theta^\epsilon_{n,\alpha}, n<\infty\}$ on $[0, \infty)$, defined by
\begin{equation*}
    \begin{aligned}
        \theta^\epsilon_\alpha &= \theta^\epsilon_{n, \alpha} \quad \text{on} \quad [n\epsilon, n\epsilon+\epsilon),\\
        \tau^\epsilon_\alpha   &= \tau^\epsilon_{n, \alpha} \quad \text{on} \quad [n\epsilon, n\epsilon+\epsilon).
    \end{aligned}
\end{equation*}
Define the \emph{real-time} interpolation $\hat{\theta}_\alpha(t)$ by
\begin{equation*}
    \hat{\theta}^\epsilon_\alpha(t)=\theta^\epsilon_{n,\alpha}, \quad t\in[\tau^\epsilon_\alpha, \tau^\epsilon_{n+1, \alpha}).
\end{equation*}

\begin{asynch_assumption}\label{assumption:asynch_unifrom_integrability_of_Y}
$\{Y^\epsilon_{n,\alpha}, \delta\tau^\epsilon_{n,\alpha}; \epsilon, \alpha, n\}$ is uniformly integrable.
\end{asynch_assumption}

\begin{asynch_assumption}\label{assumption:asynch_E_Y}
There are real-valued functions $g^\epsilon_{n,\alpha}(\cdot)$  are continuous, uniformly in $n$,  $\epsilon$ and random variables $\beta^\epsilon_{n,\alpha}$, such that
\begin{equation}
    E^\epsilon_{n,\alpha} Y^\epsilon_{n,\alpha}=g^\epsilon_{n,\alpha}(\hat{\theta}^\epsilon(\tau^{\epsilon,-}_{n+1,\alpha}), \xi^\epsilon_{n,\alpha})+\beta^\epsilon_{n,\alpha},
\end{equation}
where
\begin{equation*}
    \{\beta^\epsilon_{n,\alpha};n,\epsilon,\alpha\}\;\text{is uniformly integrable}.
\end{equation*}
\end{asynch_assumption}

\begin{asynch_assumption}\label{assumption:asynch_E_beta=0}
$\lim_{m,n,\epsilon}\frac{1}{m}\sum^{n+m-1}_{i=n}E^\epsilon_{n,\alpha}\beta^\epsilon_{i,\alpha}=0$ in mean.
\end{asynch_assumption}

\begin{asynch_assumption}\label{assumption:asynch_E_delta_tau}
There are strictly positive measurable functions $u^\epsilon_{n,\alpha}(\cdot)$, such that
\begin{equation}
    E^{\epsilon,+}_{n,\alpha}\delta\tau^\epsilon_{n+1,\alpha} = u^\epsilon_{n+1,\alpha}(\hat{\theta}^\epsilon(\tau^\epsilon_{n,\alpha}), \psi^\epsilon_{n+1,\alpha}).
\end{equation}
\end{asynch_assumption}

\begin{asynch_assumption}\label{assumption:asynch_g_continuous}
$g^\epsilon_{n,\alpha}(\cdot, \xi)$ is continuous in $\theta$, uniformly in $n$, $\epsilon$ and in $\xi\in A$.
\end{asynch_assumption}

\begin{asynch_assumption}\label{assumption:asynch_u_continuous}
$u^\epsilon_{n,\alpha}(\cdot, \psi)$ is continuous in $\theta$, uniformly in $n$, $\epsilon$ and in $\psi\in A^+$.
\end{asynch_assumption}

\begin{asynch_assumption}\label{assumption:asynch_tightness}
The set $\{\xi^\epsilon_{n,\alpha}, \psi^\epsilon_{n,\alpha}; n,\alpha,\epsilon\}$ is tight.
\end{asynch_assumption}

\begin{asynch_assumption}\label{assumption:uniform_integrablity_g_u}
For each $\theta$
\begin{equation}
    \{g^\epsilon_{n,\alpha}(\theta, \xi_{n,\alpha}), u^\epsilon_{n,\alpha}(\theta, \psi^\epsilon_{n,\alpha}); \epsilon, n\}
\end{equation}
is uniformly integrable.
\end{asynch_assumption}

\begin{asynch_assumption}\label{assumption:asynch_converging_g}
There exists a continuous function $\bar{g}_\alpha(\cdot)$, such that for each $\theta \in H$, we have
\begin{equation*}
    \lim_{m,n,\epsilon} \frac{1}{m}\sum^{n+m+1}_{i=n} E^\epsilon_{n,\alpha}[g^\epsilon_{i,\alpha}(\theta, \xi^\epsilon_{i,\alpha}) - \bar{g}_\alpha(\theta)]I_{\{\xi^\epsilon_n\in A\}}=0
\end{equation*}
in probability, as $n$ and $m$ go to infinity and $\epsilon\rightarrow0$.
\end{asynch_assumption}

\begin{asynch_assumption}\label{assumption:asynch_converging_update_rate}
There are continuous, real-valued, and positive functions $\bar{u}_\alpha(\cdot)$, such that for each $\theta \in H$:
\begin{equation*}
    \lim_{m,n,\epsilon} \frac{1}{m}\sum^{n+m+1}_{i=n} E^{\epsilon,+}_{n,\alpha}[u^\epsilon_{i+1,\alpha}(\theta, \psi^\epsilon_{i+1,\alpha}) - \bar{u}_\alpha(\theta)]I_{\{\psi^\epsilon_n\in A^+\}}=0
\end{equation*}
in probability, as $n$ and $m$ go to infinity and $\epsilon\rightarrow0$.
\end{asynch_assumption}

\vspace{5mm}
\begin{theorem}[see Theorem 3.3 and 3.5 of Ch. 12 in ~\cite{kushner2003stochastic}]\label{thm:asynch_estimator} Assume {\it\textbf{A}~\ref{assumption:asynch_unifrom_integrability_of_Y} - \textbf{A}~\ref{assumption:asynch_converging_update_rate}} hold. Then
\begin{equation*}
    \{\theta^\epsilon_\alpha(\cdot), \tau^\epsilon_\alpha(\cdot), \hat{\theta}^\epsilon_\alpha(\cdot), N^\epsilon_\alpha(\cdot), \alpha \leq r\}
\end{equation*}
is tight in $D^{4r}[0,\infty)$. Let $\epsilon$ index a weakly convergent subsequence, whose weak sense limit we denote by
\begin{equation*}
    (\theta^\epsilon_\alpha(\cdot), \tau^\epsilon_\alpha(\cdot), \hat{\theta}^\epsilon_\alpha(\cdot), N^\epsilon_\alpha(\cdot), \alpha \leq r).
\end{equation*}
Then the limits are Lipschitz continuous with probability 1 and
\begin{equation}
    \theta_\alpha(t)=\hat{\theta}_\alpha(\tau_\alpha(t)), \quad \hat{\theta}_\alpha(N_\alpha(t)),
\end{equation}
\begin{equation}
    N_\alpha(\tau_\alpha(t))=t.
\end{equation}
Moreover,
\begin{equation*}
    \tau_\alpha(t)=\int^t_0 \bar{u}_\alpha(\hat{\theta}(\tau_\alpha(s))ds,
\end{equation*}
\begin{equation*}
    \dot{\theta}_\alpha(t) = \bar{g}_\alpha(\hat{\theta}(\tau_\alpha(t))) + z_\alpha(t),
\end{equation*}
\begin{equation}\label{eq:mean_ODE_asynch}
    \dot{\hat{\theta}}_\alpha = \frac{\bar{g}_\alpha(\hat{\theta})}{\bar{u}_\alpha(\hat{\theta})} + \hat{z}_\alpha, \quad \alpha = 1, \dots, r,
\end{equation}
where  $z_\alpha$ and $\hat{z}_\alpha$ serve the purpose of keeping the paths in the interval $[a_\alpha, b_\alpha]$. On large intervals $[0,T]$, and after a transient period, $\hat{\theta}^\epsilon(\cdot)$ spends nearly all of its time (the fraction going to 1 as $\epsilon \rightarrow 0$) in a small neighborhood of $L_H$.
\end{theorem}
\begin{remark}
For decreasing step size, e.g. $\epsilon_n = 1/n^a, a\in(0,1]$, Theorem 4.1 of Ch. 12 in ~\cite{kushner2003stochastic} state that the same results in Theorem 3.5 of Ch. 12 in ~\cite{kushner2003stochastic} holds under the same assumptions (see the comment on the step-size sequence in~\cite[p.426]{kushner2003stochastic}). 
\end{remark}

\subsubsection{Convergence of the Q estimation using  stochastic approximation}
Next we state the main result of this work: the convergence of the Q estimation using state prediction. The recursive estimator of $Q^*(s,a)$, defined in the previous section, is written in the following stochastic approximation form~\cite{kushner2003stochastic}:
\begin{equation}\label{eq:HMM_based_Q_learning}
    Q_{{n+1},\alpha} = \Pi_{B_Q} \left[ Q_{n,\alpha} + \epsilon_n Y_{n, \alpha} \right],
\end{equation}
where $\alpha$ denotes indices of the parameter of $Q$, to be updated, and depends on the current action $a_n$,
{\small
\begin{equation}\label{eq:G_HMM}
\begin{aligned}
&Y_{n, \alpha}
=
G_\alpha(Q_n, \xi_n)
=
\\
&        \begin{bmatrix}
            \sum_{j}^I\hat{p}_n(1,j)\left(r_n + \gamma \max_{a'}q_{n}(j,a') - q_{n}(1, a_n) \right) \\
            \sum_{j}^I\hat{p}_n(2,j)\left(r_n + \gamma \max_{a'}q_{n}(j,a') - q_{n}(2, a_n) \right) \\
            \vdots \\
            \sum_{j}^I\hat{p}_n(I,j)\left(r_n + \gamma \max_{a'}q_{n}(j,a') - q_{n}(I, a_n) \right)
        \end{bmatrix},
\end{aligned}
\end{equation}
}
while $\xi_n$ denotes the estimated state transitions $\hat{p}(i, j)$ for all $i,j \in \mathcal{S}$ calculated in~\eqref{eq:state_transition_est}.
Now we verify {\it\textbf{A}~\ref{assumption:asynch_unifrom_integrability_of_Y} - \textbf{A}~\ref{assumption:asynch_converging_update_rate}} for the Q-function estimator in~\eqref{eq: Q_estimator_with_HMM}.

For {\it\textbf{A}~\ref{assumption:asynch_unifrom_integrability_of_Y}}, we need to show that $Y^\epsilon_{n,\alpha}=G_\alpha(q_n,\xi_n)$ in~\eqref{eq:G_HMM} is uniformly integrable. Most terms in $G_\alpha(\cdot)$ are bounded, $\hat{p}(i,j)\in[0,1]$, $q_n(s,a)$ is bounded due to the projection $\Pi_{B_Q}$, $r_n$ is the sample of $R(s_n, a_n)= r(s_n,a_n) + \delta$, where $\delta$ is i.i.d. normal distributed random variable as defined in the POMDP model. Due to the normal distribution and the bounded $q_n(\cdot)$  \& $\hat{p}(\cdot)$, we know that $P{|Y_n| < \infty}=1$. Hence, $Y_n$ is uniformly integrable, i.e. $\lim_{K\rightarrow\infty }\sup_n E|Y_n|I_{\{|Y_n|\geq K\}}=0$. The we need to show that $\delta\tau^\epsilon_{n,\alpha}$ is uniformly integrable. According to Assumption\ref{assumption:nice_behavioral_policy}, the probability of not choosing an action for infinitely long is zero. So $\delta\tau^\epsilon_{n,\alpha}$ is uniformly integrable, i.e. $\lim_{K\rightarrow\infty }\sup_n E|\delta\tau^\epsilon_{n,\alpha}|I_{\{|\delta\tau^\epsilon_{n,\alpha}|\geq K\}}=0$. Hence, {\it\textbf{A}\ref{assumption:asynch_unifrom_integrability_of_Y}} holds.

For {\it\textbf{A}~\ref{assumption:asynch_E_Y}}, write $E^\epsilon_{n,\alpha}Y^\epsilon_{n,\alpha} = g^\epsilon_{n,\alpha}(\hat{\theta}^\epsilon(\tau^{\epsilon,-}_{n+1,\alpha}), \xi^\epsilon_{n,\alpha})+\beta^\epsilon_{n,\alpha}$ with the Q-function estimator in~\eqref{eq:HMM_based_Q_learning} as
{\small
\begin{equation}
\begin{aligned}\label{eq:E_Y}
    &E^\epsilon_{n,\alpha}Y^\epsilon_{n,\alpha}
    \\
    &=
    \begin{bmatrix}
    \sum_{j}^I\hat{p}_n(1,j)\left(r_n + \gamma \max_{a'}q_{n}(j,a') - q_{n}(1, a_n) \right) \\
    \vdots \\
    \sum_{j}^I\hat{p}_n(I,j)\left(r_n + \gamma \max_{a'}q_{n}(j,a') - q_{n}(I, a_n) \right)
    \end{bmatrix}
    \\
    &=g^\epsilon_{n,\alpha}(\hat{\theta}^\epsilon(\tau^{\epsilon,-}_{n+1,\alpha}), \xi^\epsilon_{n,\alpha})+0,
\end{aligned}
\end{equation}
}
where $\xi^\epsilon_{n,\alpha}=(r_n, a_n, (\hat{p}_n(i,j)))$ and $\theta^\epsilon$ corresponds to $q(i,a)$. From the above equation, it is easy to see that $g^\epsilon_{n,\alpha}(\cdot)$ is real valued continuous function, and $\beta^\epsilon_{n,\alpha}=0$, so it is trivially uniformly integrable.

{\it\textbf{A}~\ref{assumption:asynch_E_beta=0}} is trivially satisfied, since $\beta^\epsilon_{n,\alpha}=0$.

For {\it\textbf{A}~\ref{assumption:asynch_E_delta_tau}}, we verify it using  \textbf{Assumption}~\ref{assumption:nice_behavioral_policy} on the behavioral policy. We use the same argument from~\cite[p.440]{krishnamurthy2002recursive}. Let $\{\psi^\epsilon_n\}$ denote the sequence of observation, which is used to generate actions by the behavior policy in \textbf{Assumption}~\ref{assumption:nice_behavioral_policy}. According to the assumption, the probability that an arbitrary chosen action can be strictly positive can be verified as follows. Suppose that there are $n_0<\infty$ and $\delta_0>0$, such that for each state pair $i,j$ we have:
\begin{equation}\label{eq:positive_policiy}
    \inf P\{\psi^\epsilon_{n+k}=j,\; \text{for some}\; k\leq n_0|\psi^\epsilon_n=i \}\geq\delta_0.
\end{equation}
Define $u^\epsilon_{n+1,\alpha}$ by
\begin{equation*}
    E^{\epsilon,+}_{n,\alpha} \delta\tau^\epsilon_{n+1,\alpha} = u^\epsilon_{n+1,\alpha},
\end{equation*}
and recall that $\delta \tau^\epsilon_{n, \alpha}$ denotes the time interval between the $n^{\rm th}$ and $(n+1)^{\rm th}$ occurrences of the action index $\alpha$. Then \eqref{eq:positive_policiy} implies that $\{\delta\tau_{n,\alpha}\}$ are uniformly bounded (but greater than 1), i.e. the expected recurrence time of each action index is finite.

Verifying {\it\textbf{A}~\ref{assumption:asynch_g_continuous}} easily follows from \eqref{eq:E_Y}. The the function in \eqref{eq:E_Y}  consists of basic operations such as addition, multiplication and $\max$ operator, which guarantee continuity of the function.

Verification of {\it\textbf{A}~\ref{assumption:asynch_u_continuous}} also follows trivially due to the fact that the behavior policy and the state transition do not depend on $\theta$, which is $q(s,a)$, since it is off-policy learning.

For {\it\textbf{A}~\ref{assumption:asynch_tightness}}, we state the definition of \emph{tightness}.
\begin{definition}[tightness of a set of random variables] Let $B$ be a metric space. Let $\mathcal{B}$ denote the minimal $\sigma$-algebra induced on $B$ by the topology generated by the metric. Let $\{A_n, n<\infty\}$ and $A$ be $B$-valued random variables defined on a probability space $(\Omega,P,\mathcal{F})$. A set $\{A_n\}$ of random variables with values in $B$ is said to be \emph{tight}, if for each $\delta>0$ there is a compact set $B_\delta \subset \mathcal{B}$, such that
\begin{equation}
    \sup_n P\{A_n \notin B_\delta\} \leq \delta.
\end{equation}
\end{definition}
Notice that
\begin{equation*}
\begin{aligned}
    \xi^\epsilon_{n,\alpha}&=\xi^\epsilon_n=(r_n, a_n, (\hat{p}_n(i,j))), \\
    \psi^\epsilon_{n,\alpha}&=\psi^\epsilon_n = o_n,
\end{aligned}
\end{equation*}
where $a_n, o_n, \hat{p}_n(\cdot)$ are bounded, and $r_n$ is the sum of bounded $r(s,a)$ and i.i.d. Gaussian noise.
Hence, the tightness (boundedness in probability) of $\{\xi^\epsilon_{n,\alpha}, \psi^\epsilon_{n,\alpha}; n,\alpha,\epsilon\}$ is straightforwardly verified.

We have checked the boundedness of $\{g^\epsilon_{n,\alpha}(\cdot)\}, \{u^\epsilon_{n,\alpha}(\cdot)\}$, when we verified {\it\textbf{A}~\ref{assumption:asynch_g_continuous}} and {\it\textbf{A}~\ref{assumption:asynch_u_continuous}} above. So uniform integrability in  {\it\textbf{A}~\ref{assumption:uniform_integrablity_g_u}} is verified.

When we verified {\textbf{C}~\ref{assumption: geometrically_ergodic_extended_MC}}, the geometric ergodicity of the extended Markov chain $\{ s_n, y_n, \hat{p}_n, \omega_n \}$ was proven. Due to the ergodicity, both $\xi^\epsilon_{n,\alpha}$ and $\psi^\epsilon_{n,\alpha}$ converge to the stationary distribution. Hence, {\it\textbf{A}~\ref{assumption:asynch_converging_g}} and {\it\textbf{A}~\ref{assumption:asynch_converging_update_rate}} hold.

Now, we have verified {\it\textbf{A}~\ref{assumption:asynch_unifrom_integrability_of_Y} - \textbf{A}~\ref{assumption:asynch_converging_update_rate}} in Theorem~\ref{thm:asynch_estimator}. Accordingly, the iterate of the estimator converges to the set of the limit points of the ODE in~\eqref{eq:mean_ODE_asynch}, and $q_n(s,a)$ converges to the solution of the following ODE:
{\small
\begin{equation*}
        \begin{bmatrix}
            \dot{q}_{1,a}\\
            \dot{q}_{2,a}\\
            \vdots     \\
            \dot{q}_{I,a}
        \end{bmatrix}
        =
        \frac{1}{\bar{u}_a}
        \begin{bmatrix}
        \sum_{j}^I\bar{p}(1,j) (\bar{r} + \gamma \max_{a'}q_{j,a'} - q_{1, a}) \\
        \sum_{j}^I\bar{p}(2,j) (\bar{r} + \gamma \max_{a'}q_{j,a'} - q_{2, a}) \\
        \vdots \\
        \sum_{j}^I\bar{p}(I,j)  (\bar{r} + \gamma \max_{a'}q_{j,a'} - q_{I, a})
        \end{bmatrix}
        +  \hat{z}_a.
\end{equation*}
}
We first ignore $\hat{z}_a$ and define the operator $F(Q) = [F_{i,a}(Q)]_{i,a}$ with
\begin{equation*}
    F_{i,a}(Q) = \sum_{j}^I\frac{\bar{p}(i,j)}{\sum_{k}^I\bar{p}(i,k)} (\bar{r} + \gamma \max_{a'}(q_{j,a'})),
\end{equation*}
where $Q= [q_{i,a}]_{i,a} = \begin{bmatrix}
    \ddots & {} &    \\
   {} & {q_{i,a} } & {}  \\
      & {} &  \ddots   \\
\end{bmatrix}$, and
\begin{equation*}
\Theta_{i,a} : = \frac{\sum_k^N {\bar p(i,k)}}{\bar u_a}.
\end{equation*}
Then, the ODE is expressed as $\dot Q =  \Theta  \circ(F(Q) - Q)$, where $\circ$ is the Hadamard product. Using the standard proof for the Q-learning convergence~\cite{bertsekas1996neuro}, we can easily prove that $F$ is a contraction in the max-norm $\|\cdot \|_\infty$. If we consider the ODE $\dot Q =  F(Q) -Q$, the global asymptotic stability of the unique equilibrium point is guaranteed by the results in~\cite{borkar1997analog}. Returning to the original ODE $\dot Q =  \Theta\circ (F(Q) - Q)$, we can analyze its stability in a similar way. Define the weighted max-norm $\| A \|_{\Theta^{-1},\infty}:= \max_{i,j} \Theta _{ij}^{-1} A_{ij} $ for a matrix $A$. Then, $\Theta  \circ F$ is a contraction with respect to the norm $\| A \|_{\Theta^{-1},\infty}$. Using this property, we can follow similar arguments of the proof of~\cite[Theorem 3.1]{borkar1997analog} to prove that the  unique fixed point $Q^*$ of $F(Q^*)= Q^*$ is a globally asymptotically stable equilibrium point of the ODE $\dot Q =  \Theta  \circ(F(Q) - Q)$.
\QEDA
\end{document}